\documentclass{article}

\usepackage{microtype}
\usepackage{graphicx}
\usepackage{subfigure}
\usepackage{booktabs} 
\usepackage{rotating}

\usepackage{hyperref}

\usepackage[preprint]{arxiv_2025}

\usepackage{amsmath}
\usepackage{amssymb}
\usepackage{mathtools}
\usepackage{amsthm}

\usepackage[capitalize,noabbrev]{cleveref}

\theoremstyle{plain}
\newtheorem{theorem}{Theorem}[section]
\newtheorem{proposition}[theorem]{Proposition}
\newtheorem{lemma}[theorem]{Lemma}
\newtheorem{corollary}[theorem]{Corollary}
\theoremstyle{definition}
\newtheorem{definition}[theorem]{Definition}

\theoremstyle{remark}

\usepackage[textsize=tiny]{todonotes}

\usepackage{abbreviations}
\usepackage{xcolor}
\usepackage{algorithm}
\usepackage{algorithmic}
\usepackage{dsfont}
\usepackage{multirow}
 \usepackage{natbib} 

\renewcommand\labelenumi{(\roman{enumi})}
\renewcommand\theenumi\labelenumi

\usepackage{wrapfig}

\title{Kernel Trace Distance: Quantum Statistical Metric between Measures through RKHS Density
Operators
}

\author{%
  Arturo Castellanos \\
  LTCI, Télécom Paris, Institut Polytechnique de Paris\\
  \texttt{arturo.castellanos@telecom-paris.fr}\\
  \And
  Anna Korba \\
  CREST, ENSAE, Institut Polytechnique de Paris\\
  \texttt{anna.korba@ensae.fr}
  \AND 
  Pavlo Mozharovskyi\\
  LTCI, Télécom Paris, Institut Polytechnique de Paris\\
  \texttt{pavlo.mozharovskyi@telecom-paris.fr} \\
  \And
  Hicham Janati \\
  LTCI, Télécom Paris, Institut Polytechnique de Paris \\
  \texttt{hicham.janati@telecom-paris.fr}
}

\begin{document}

\maketitle

\begin{abstract}
Distances between probability distributions are a key component of many statistical machine learning tasks, from two-sample testing to generative modeling, among others. We introduce a novel distance between measures that compares them through a Schatten norm of their kernel covariance operators. We show that this new distance is an integral probability metric that can be framed between a Maximum Mean Discrepancy (MMD) and a Wasserstein distance. In particular, we show that it avoids some pitfalls of MMD, by being more discriminative and robust to the choice of hyperparameters. Moreover, it benefits from some compelling properties of kernel methods, that can avoid the curse of dimensionality for their sample complexity. We provide an algorithm to compute the distance in practice by introducing an extension of kernel matrix for difference of distributions that could be of independent interest. Those advantages are illustrated by robust approximate Bayesian computation under contamination as well as particle flow simulations.
\end{abstract}

\section{INTRODUCTION}

Statistical distances are ubiquitous in the fundamental theory of machine learning and serve as the backbone of many of its applications, such as: discriminating between the generative model and real data in Generative Adversarial Networks (GAN)~\citep{goodfellow2014generative,arjovsky2017wasserstein, li2017mmd,genevay2018learning,birrell2022f}, testing whether a dataset is close to another (two-sample test)~\citep{eric2007testing,gretton2012kernel,hagrass2024spectral} or to a particular distribution (goodness-of-fit test), as well as acting as an objective loss function in particle gradient flows~\citep{arbel2019maximum,feydy2019interpolating,korba2021kernel,hertrich2023generative,neumayer2024wasserstein,chen2024regularized}, or in minimum distance estimators~\citep{wolfowitz1957minimum,basu2011statistical}.

A class of distances between probability distributions, called Integral Probability Metrics (IPM)~\citep{muller1997integral}, is 
defined by measuring the supremum of difference of integrals over a function space. 
It comprises many popular metrics such as the Total Variation distance, Wasserstein-1 distance 
and the Maximum Mean Discrepancy (MMD)~\citep{gretton2012kernel} also known as quadratic distance~\citep{lindsay2008quadratic}.
IPMs' theoretical properties were largely investigated in the literature, such as their statistical  convergence rate~\citep{sriperumbudur2010non}, concentration for inference using ABC~\citep{legramanti2022concentration}, PAC-Bayes bounds~\citep{amit2022integral}, as well as adversarial interpretations~\citep{husain2022adversarial}. For instance, the MMD enjoys a fast statistical  convergence rate of $O(n^{-\frac{1}{2}})$ while the Wasserstein distance 
suffers from the curse of dimensionality with a rate no better than $\Theta(n^{-\frac{1}{d}})$~\citep{kloeckner2012approximation}. 
One could wonder: \emph{how large such a function space could be before the curse of dimensionality kicks in?} 
In this work, we theoretically investigate how to get closer to
such frontier by 
defining
an extended family of kernel distances, that write as novel IPM whose dual function space is larger than the one of MMD.

 Kernel methods 
 allow to represent a distribution by a vector by associating to a datapoint $x$ a 
 feature map image $\varphi(x)$ in a Hilbert space, and by doing so, embed in a linear way a distribution $\mu$ to what is called a  \emph{(kernel) mean embedding} $\mathbb{E}_{X\sim\mu}[\varphi(X)]=\int
\varphi(x) d\mu(x)$.
However mean embeddings for different distributions may have different ``energies", i.e., squared Hilbert norms, which may lead to several pitfalls of MMD. 
In quantum information theory~\citep{watrous2018theory}, a similar idea to mean embedding is called superposition. The quantum equivalent of a datapoint or deterministic Dirac distribution is called a \emph{pure state} and is a projector of rank 
and trace one, that could be denoted $vv^*$ (or $|v\rangle \langle v|$) for a unit vector $v$. Its 
analog for a general probability distribution 
is called a \emph{mixed state} and is the superposition $\sum_v p(v) |v\rangle \langle v|$ where $p(v)$ are probabilities. A non-trivial mixed state can hardly be confused with a pure state as a linear combination of different projectors is 
of higher rank than 1: using projecting operators instead of the vectors themselves makes the \emph{linearity less ``trivial"}. As those positive definite operators can be diagonalised, by using always the same orthogonal basis and studying the eigenvalues, we recover classical probabilities, and as such we can see quantum probabilities as their extension.
Recently, the work of~\citet{bach2022} introduced a novel divergence between probability distributions, by 
plugging a kernel operator embedding of the distributions (which are also positive definite operators) in the Von Neumann relative entropy from quantum information theory (i.e., a Kullback-Leibler divergence between positive Hermitian operators), 
and whose statistical and geometrical properties were investigated more in depth in
~\citet{chazal2024statistical}. 
Instead of considering a divergence on such operators, here we propose to draw inspiration from quantum statistical metrics, which enjoy nice geometrical properties such as the triangle inequality. Two of them are well-known and mutually bounding: the Bures 
metric, 
and the trace distance, on which we focus here, and 
which is derived from a (Schatten) norm. 

\paragraph{Related works}

The kernelised version of Bures metric, i.e., a Bures metric between kernel covariance operators, has been studied for instance in~\citet{KWD,zhang2019optimal}.
The closest work to ours is the one by~\citet{mroueh2017mcgan}. They consider a similar metric to ours, i.e. the trace distance, that they refer to as Covariance Matching IPM. It shares the same dual writing as the metric we consider, yet, in that work, the dual problem is solved through a numerical program involving neural networks that approach kernel features. Hence, they compute an approximate version of their target metric. 
In contrast, we use kernel features directly in the dual formulation, and derive a closed-form for the metric leveraging a kernel trick. Moreover, we provide theoretical guarantees regarding this metric and investigate different numerical applications than the one of the GAN considered in ~\citet{mroueh2017mcgan}.

\paragraph{Contributions}
Our main contributions can be summarized as follows: 
\begin{enumerate}
    \item Inspired by quantum statistics, we introduce a novel distance between probability distributions called \emph{kernel trace distance} ($\dist$).
    \item We show that $\dist$ is an IPM and illustrate several of its theoretical properties, mainly: 
    a direct comparison to MMD,
    robustness to contamination,
    and statistical convergence rates that do not depend on the dimension. 
    \item We showcase how to compute $\dist$ and illustrate its practical performance on particle gradient flows and Approximate Bayesian Computation (ABC).
\end{enumerate}

\paragraph{Organisation of the paper}
In section~\ref{sec:definition}, we provide some background on quantum statistical distances and introduce $\dist$. In section~\ref{sec:motivation}, we explain further the motivation to introduce $\dist$, notably by comparing it with the other distances, MMD in particular. We show in section~\ref{sec:stat}, under some eigenvalue decay rate assumptions, convergence rates  that do not depend on the dimension, as well as robustness. In section~\ref{sec:computation}, we explain how to compute $\dist$. 
Finally, we illustrate our findings by experiments in section~\ref{sec:exp}.
\section{Kernel Trace Distance}
\label{sec:definition}

For a positive semi-definite kernel $k : \mathcal{X} \times \mathcal{X} \to \mathbb{R}$, its  RKHS $\mathcal{H}$ is a Hilbert space of real-valued functions with inner product $\langle \cdot,\cdot \rangle_{\mathcal{H}}$ and norm $\Vert \cdot \Vert_{\mathcal{H}}$. It is associated with a feature map $\varphi:\mathcal{X}\to \mathcal{H}$ such that $k(x,y)=\langle \varphi(x),\varphi(y)\rangle_{\mathcal{H}}$. 
We denote  $\mathcal{L}(\mathcal{H})$ the space of bounded linear operators from $\mathcal{H}$ to itself. 
For a vector $v\in\mathcal{H}$, $v^*$ denotes its dual linear form defined by $v^*(w)=\langle v,w \rangle$ for any $w\in \mathcal{H}$. For an operator $T\in\mathcal{L}(\mathcal{H})$, $T^*$ is its adjoint.  $\SL{p}\cdot\SR{p}$ denotes the $p$-Schatten norm explicited below.

\begin{assumptionp}{0}
In the whole paper, we restrict ourselves to 
the setting of 
a completely separable set $\mathcal{X}$, endowed with a Borel $\sigma$-algebra, and a separable RKHS $\mathcal{H}$ of real-valued functions on $\mathcal{X}$, with a bounded continuous strictly positive kernel.
\end{assumptionp}

\subsection{Background}

\paragraph{RKHS density operators~\citep{bach2022}.}   Let $\mu$ a measure on $\mathcal{X}$. 
    Define $\Phi$ the kernel covariance operator embedding as:
    \begin{align}\label{eq:densdef}
        \Phi : \mu \mapsto \Sigma_\mu = \int_\mathcal{X} \varphi(x)\varphi(x)^* d\mu(x).
    \end{align}
    We will call  $\Sigma_\mu$ the RKHS density operator of $\mu$, in reference to  the wording of density operator in quantum information theory: this is to insist that $\Sigma_\mu$ is an embedding in itself (with feature map $\varphi(\cdot)\varphi(\cdot)^*$), 
    rather than just the covariance of a mean embedding with feature map $\varphi$. The operator $\Sigma_\mu$ is self-adjoint, and positive semidefinite when $\mu$ is a probability measure.
    To keep the analogy with quantum density operators, similarly to~\citet{bach2022}, we consider kernels respecting the property:
    \begin{assumptionp}{1}\label{assum:kernel1}
    $    \forall x\in\mathcal{X}, \;k(x,x)=1.$ 
    \end{assumptionp}
    to ensure $\Tr \Sigma_\mu = 1$ (as in the sum of all probabilities equals one).
    If $\forall x\in\mathcal{X}, \;k(x,x)=M$ for a non-zero constant $M\neq 1$, it is will be easy to generalize many of our results later by dividing by $M$, so this assumption is not too restrictive.
    If the kernel does not verify Assumption~\ref{assum:kernel1} but is strictly positive, it is could also be normalised using $\Tilde{k}(x,y) = \frac{k(x,y)}{\sqrt{k(x,x)k(y,y)}}$ instead.

\paragraph{Schatten norms.}
We now provide some background on Schatten norms~\citep{simon2005trace}. For an operator $T\in\mathcal{L}(\mathcal{H})$ and $p \in [1,\infty)$, the $p$-Schatten norm is defined as 
$    \SPL T \SPR = (\Tr(|T|^p))^{1/p}$
where $|T| = \sqrt{T^* T}$.
If $T$ is compact, 
this can be rewritten as the $p$-vectorial norm of the singular values of $T$.
It also admits a dual definition, denoting $q$ such that $1/p + 1/q = 1$:
\begin{align}\label{eq:schattendual}
    \SPL T \SPR = \sup_{U \in \mathcal{L(H)}, ||U||_q = 1} \langle U,T \rangle
\end{align}
where the inner product is $\langle U,T \rangle = \Tr(U^*T)$.

The Schatten 2-norm is the Hilbert-Schmidt norm with respect to this inner product:  $||T||_2 = \sqrt{\Tr(T^*T)}$. 
Then, the Schatten $\infty$-norm 
is the operator norm : $\SL{\infty}T\SR{\infty} = \sup_{x  \in \mathcal{H}\backslash 0 } \frac{||Tx||_\mathcal{H}}{||x||_\mathcal{H}}$ i.e., the maximum of the singular values of the operator in absolute value. 
 We have the following inequalities:
\begin{itemize}
    \item For $1\leq p \leq q \leq \infty$:
    $\forall T \in \mathcal{L(H)}, \SL{1} T \SR{1} \geq \SL{p}T\SR{p} \geq \SL{q}T\SR{q} \geq \SL{\infty}T\SR{\infty}.$
    \addtocounter{equation}{1}
    \item $\forall T,S \in \mathcal{L(H)}, \SL{1} TS \SR{1}  \leq \SL{2} T \SR{2} \SL{2} S \SR{2}\, .$ \qquad \qquad \qquad \qquad \qquad \qquad \qquad \qquad \quad \quad \quad (\theequation)
    \item From this, it can be deduced taking \( T \) as the identity operator, for $\mathcal{H}$ of finite dimension:
        \begin{align}\label{ineq:dim}
        \forall S \in \mathcal{L(H)}, \SL{1} S \SR{1} \leq \sqrt{\dim(\mathcal{H})} \SL{2} S \SR{2}.
    \end{align}
\end{itemize}

\subsection{Definition}
In quantum information theory, the trace distance is a mathematical tool that can be used to compare density operators by measuring the Schatten 1-norm of their difference.
Inspired by this, we define:
\begin{definition}
The \textbf{\emph{kernel trace distance}}
between two probability measures $\mu,\nu$ on $\mathcal{X}$ is defined as:
\begin{equation*}
\dist(\mu,\nu) = \SL{1} \Sigma_\mu - \Sigma_\nu \SR{1}.
\end{equation*}
\end{definition}
We will also relate it to other distances such as:
\begin{itemize}
    \item Wasserstein distances   ~\citep{villani2009optimal}:
    \begin{equation*}
        W_d(\mu,\nu) = \inf_{\pi \in \Pi(\mu,\nu)} \iint d(x,y) \mathrm d\pi(x,y)
\end{equation*}
   where  $d:\mathcal{X}\times \mathcal{X}\to \mathbb{R}^+$  is a cost and $\Pi(\mu,\nu)$ denotes all the possible couplings between $\mu$ and $\nu$. The Wasserstein-$p$ distance is obtained by replacing $d$
by 
its power $d^p$ 
in the integral and taking the $p$-root of the whole expression.
    
    \item The Bures 
distance~\citep{bhatia2019bures} on positive definite matrices $A$:
    \begin{equation*}
        d_{BW}(A,B) = \sqrt{\Tr A + \Tr B - 2 F(A,B)}
    \end{equation*}
    where $F(A,B)= \Tr (A^{1/2} B A^{1/2})^{1/2}$ is called the fidelity. It coincides with the Wasserstein-2 distance between two normal distributions (also called Bures-Wassertein distance) with identical  mean, and different covariances $A$ and $B$. The formula can be extended to operators with finite traces.
    \item The Kernel Bures distance~\citep{zhang2019optimal} is defined as:  
    \begin{equation*}
        d_{KBW}(\mu,\nu) = d_{BW}(\Sigma_\mu,\Sigma_{\nu}).
    \end{equation*}
    
    \item The Total Variation 
    is a special case of the Wasserstein distance where the cost is $d:(x,y)\mapsto 1_{x=y}$ and can be expressed as:
    \begin{equation*}
        ||\mu-\nu||_{TV} = \frac{1}{2} \int_\mathcal{X} |\mu(x)-\nu(x)|dx
    \end{equation*}

    \item The Maximum Mean Discrepancy~\citep{gretton2012kernel}:
    \small
    \begin{equation*}
        \MMD(\mu,\nu) = \left|\left|\int_\mathcal{X} k(x,\cdot)\mu(x)dx -\int_\mathcal{X} k(x,\cdot)\nu(x)dx\right|\right|_\mathcal{H}
    \end{equation*}
    \normalsize

    \item Integral Probability Metrics  (IPM)~\citep{muller1997integral}   defined as:
\begin{equation*}
            d(\mu,\nu)=\sup_{f \in \mathcal{F}} \{|\mathbb{E}_{X\sim\mu}[f(X)]-\mathbb{E}_{X\sim\nu}[f(X)]|\}
    \end{equation*}
    where the function space $\mathcal{F}$ is rich enough to make this expression a metric. The Wasserstein-1 distance, the TV and MMD are IPMs (with $\mathcal{F}$ being 1-Lipschitz functions w.r.t. $\|\cdot\|$, functions with values in [-1,1], and a RKHS unit ball respectively).
    
\end{itemize}

\begin{proposition}
If $k^2$ is characteristic i.e $\Phi$ is injective, 
$\dist$ and $d_{KBW}$ are metrics.
\end{proposition}

\begin{proof}
Symmetry, non-negativity, triangle inequality and 
$\dist(\mu,\mu)=0$ (resp. $d_{KBW}(\mu,\mu)=0$)
are naturally inherited from the Schatten norm on operators for $\dist$ and from the standard Bures-Wasserstein distance for $d_{KBW}$. Then, 
as $\dist(\mu,\nu)=0$ (resp. $d_{KBW}(\mu,\nu)=0$) implies $\Sigma_\mu=\Sigma_\nu$, injectivity of $\Phi$ enforces $\mu=\nu$.
\end{proof}  

Examples of characteristic kernels are the family of Gaussian 
kernels, whose squared kernel also belong to, modulo a change of parameter. On compact set, a sufficient condition for characteristicity is universality~\citep{steinwart2001influence}, see for instance~\citet{bach2022}.
\subsection{Computation for discrete measures}\label{sec:computation}
As interesting, i.e. expressive RKHS are often of infinite dimension, computations with kernel methods relies on the so-called ``kernel trick", reducing computation on the empirical kernel matrix (Gram matrix of two sets of samples using the kernel inner product) which is of finite dimension. It is well-known that the spectrum of the covariance operator $\Sigma_{\mu_n}$ are the ones of the kernel Gram matrix $(k(x_i,x_j))_{i,j=1}^n$divided by the number of samples~\citep[Proposition 6]{bach2022}. Here, we generalise the concept for differences of distributions. 

First, notice that $\Sigma_{\mu_n}-\Sigma_{\nu_m}=\Sigma_{\mu_n-\nu_m}$, which incites us to consider the samples from each distribution altogether.
We denote without duplicates $(z_k)_{k=1,...,r}$ the samples in the union of the sample sets $X,Y$ (corresponding respectively to distributions $\mu_n,\nu_m$), where $r$ is the number of distinct elements in $X,Y$. 
We note $Z = [\Tilde{\varphi}(z_k)]_{k=1...r}$ the column of 
vectors in $\mathcal{H}$ where
$\Tilde{\varphi}(z_k)= \sqrt{(\mu_n-\nu_m)(\{z_k\})}\varphi(z_k)$ if $(\mu_n-\nu_m)(\{z_k\})\geq 0$, $\Tilde{\varphi}(z_k)=
         i\sqrt{|(\mu_n-\nu_m)(\{z_k\})|}\varphi(z_k)$ else. 

We can see $Z$ by a slight abuse of notation as the linear map $Z: \mathcal{H} \to \mathbb{C}^r, v \mapsto [\langle \Tilde{\varphi}(z_1), v\rangle,...,\langle \Tilde{\varphi}(z_r), v\rangle]$ and by duality $Z^*$ (real not Hermitian adjoint) would be the linear map $Z^*:\mathbb{C}^r \to \mathcal{H}, u \mapsto \sum_{i=1,...,r} u_i \Tilde{\varphi}(z_i)$.

Then we define the \emph{difference kernel matrix} as $K = Z^* Z$. Typically, in case where all samples are distinct, $X \cap Y = \emptyset$ and $(\mu_n-\nu_m)(\{z_k\}) = \mu_n(\{z_k\}) = 1/n$ for samples $z_k\in X$ from $\mu_n$ and 
$(\mu_n-\nu_m)(\{z_k\}) = -\nu_m(\{z_k\}) = 1/m$ for samples $z_k\in Y$ from $\nu_m$, then
\begin{equation*}
    K=\left[
\begin{array}{c|c}
\frac{1}{n} K_{XX} & \frac{i}{\sqrt{mn}}K_{XY} \\ \hline
\frac{i}{\sqrt{mn}}K_{YX} & -\frac{1}{m} K_{YY} 
\end{array}\right]
\end{equation*}

where $K_{XX},K_{YY},K_{YX},K_{XY}$ are the usual kernel Gram matrices. Other cases are similar, adjusting the probability weights on rows and columns.

\begin{proposition}\label{prop:computation}
    Assume the kernel is such that for any family $(x)$ of distinct elements of $\mathcal{X}$, $(\varphi(x))$ is linearly independent.
    The difference kernel matrix $K$ as defined just above and $\Sigma_{\mu_n-\nu_m}$ have the same eigenvalues, whose Schatten $1$-norm is $d_{KT}(\mu_n,\nu_m)$.
\end{proposition}
The proof of \Cref{prop:computation} is deferred to \Cref{sec:proof_computation}. The condition is verified by the Gaussian kernel and more generally it is equivalent to the kernel being strictly positive.
It is sufficient to get the eigenvalues by either Autonne-Takagi factorisation~\citep{autonne1915matrices,takagi1924algebraic}, Schur or Singular Value decomposition, and compute their 1-norm.
This SVD
is of complexity $O(r^3)$ in general. 

\section{Discriminative properties}\label{sec:motivation}
In this section, we study the discriminative properties of the $d_{KT}$ distance and how it relates to alternative distances between distributions introduced previously. 

\subsection{Comparison with other distances}\label{subsec:firstprop}
We first show that our novel distance $\dist$ belongs to the family of Integral Probability Metrics (IPM). 

\begin{proposition}\label{thm:IPM}
\item[] \vspace{-0.2cm}
\begin{enumerate}
    \item $\dist$ is an $\IPM$ with respect to the function space 
    \small
    $\mathcal{F}_1 = \{f:~x~\mapsto~\varphi(x)^*U\varphi(x) | U\in\mathcal{L}(\mathcal{H}),
    ||U||_\infty = 1  \}$. \label{item:1}
    \normalsize
    \item[] Moreover if Assumption~\ref{assum:kernel1} is verified:
    \item functions in $\mathcal{F}_1$ have values in $[-1,1]$, and \label{item:2}
    \item verify the following ``Lipschitz" property: 
    $     \forall x,y \in \mathcal{X},\; |f(x)-f(y)| \leq 2||\varphi(x) - \varphi(y) ||_\mathcal{H}.$
 \label{item:3}
\end{enumerate}
\end{proposition}
The proof of \Cref{thm:IPM} is deferred to \Cref{sec:proof_IPM}. Since the TV distance is an IPM with respect to functions bounded by 1, we have the following corollary:
\begin{corollary}
$\dist(\mu,\nu) \leq || \mu - \nu||_{TV}$.
\end{corollary}
We also have a direct comparison between $d_{KT}$ and a MMD. 
\begin{lemma}\label{prop:MMDschatten2}
The Schatten 2-norm of the difference of the RKHS density operators of two probability distributions $\mu,\nu$ on $\mathcal{X}$ can be identified to their Maximum Mean Discrepancy using the kernel $k^2$:
\begin{equation*}
     \STL\Sigma_\mu-\Sigma_\nu\STR = \MMD_{k^2}(\mu,\nu)   
\end{equation*}
Consequently, since $d_{KT}$ is a Schatten 1-norm of this difference, $\MMD_{k^2}(\mu,\nu) \leq \dist(\mu,\nu)$.
\end{lemma}
This follows mainly from the fact that $\langle \Sigma_\mu, \Sigma_\nu\rangle = 
\int_\mathcal{X} \int_\mathcal{Y} k(x,y)k(x,y) \mu(x)  \nu(y)dx dy$ (see  Appendix~\ref{appendsssec:firstprop}). 
Finally, we can relate $d_{KT}$ to some Wasserstein distance. 
Denoting $c_k(x,y) =  ||\varphi(x) - \varphi(y) ||_\mathcal{H} = \sqrt{2(1-k(x,y))}$ a cost defined from the kernel $k$, 
and applying
the Lipschitz property of Theorem~\ref{thm:IPM}, we get the following:

\begin{corollary}\label{cor:wasserineq}
If Assumption~\ref{assum:kernel1} is verified,
   $ \dist(\mu,\nu) \leq 2 W_{c_k}(\mu,\nu).$
Furthermore, using the Gaussian kernel with parameter $\sigma$,
\begin{align*}
    \dist(\mu,\nu) \leq 2 W_{c_k}(\mu,\nu)\leq \frac{2}{\sigma} W_{||.||}(\mu,\nu). 
\end{align*}
\end{corollary}
The last remark is due to the fact that the Wasserstein-1 distance is an IPM defined by the functions which are 1-Lipschitz w.r.t. $\|\cdot\|$, 
and for the Gaussian kernel $k(x,y) = e^{-\frac{||x-y||^2}{2\sigma^2}}$, we have $
    c_k(x,y)
    \leq \frac{||x-y||}{\sigma}$. See Appendix~\ref{appendsssec:firstprop} for full proof.


Finally our novel distance can be related to other kernelized quantum divergences. Some well-known inequality in quantum information theory relating the trace distance and the fidelity is the following 
\citet{FuchsDeGraaf} inequality : 
\begin{align}\label{ineq:FvdG}
    2(1 - F(A,B) )\leq ||A-B||_1 \leq 2\sqrt{1 - F(A,B)^2}
\end{align}
which translates as upper and lower bounds on $\dist$ with respect to $d_{KBW}$ (see proof in Appendix~\ref{appendsssec:firstprop} using Assumption~\ref{assum:kernel1}):
\begin{align}\label{ineq:d1kbw}
    d_{KBW}(\mu,\nu)^2 \leq \dist(\mu,\nu) 
    \leq 2 d_{KBW}(\mu,\nu)
\end{align}
Let $D_{\mathrm{KL}}(A|B)=\Tr(A(\log A-\log B))$ the quantum relative entropy. The Kernel-Kullback-Leibler ($\mathrm{KKL}$) divergence introduced in \citet{bach2022} is defined as the latter applied to the density operators of two distributions $\mu,\nu$ on $\mathcal{X}$ (in particular, it is infinite if $\mu$ is not absolutely continuous w.r.t. $\nu$). Thanks to the (quantum) Pinsker's inequality, we have then:
   $\frac{1}{2}\dist(\mu,\nu)^2 \leq D_{KL}(\Sigma_\mu | \Sigma_\nu):=\mathrm{KKL}(\mu|\nu).$
Hence, our distance can be framed within several well-known alternative discrepancies.


\subsection{Normalized energy}
\label{subsec:mmdpitfalls}

\begin{wrapfigure}{r}{0.4\textwidth}
\includegraphics[width=0.85\linewidth]{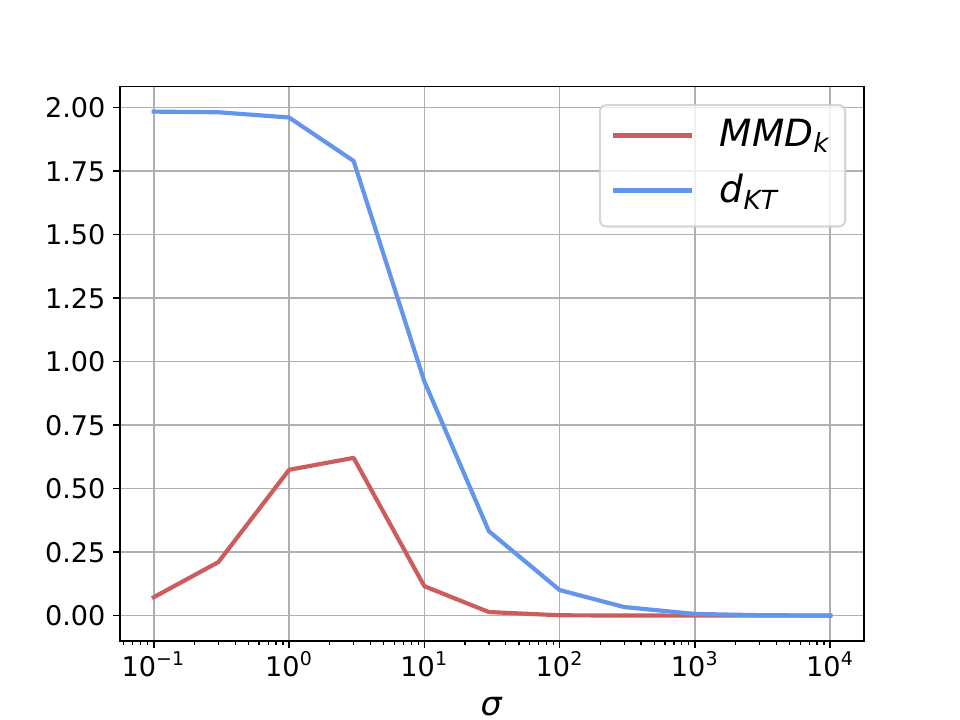}
         \caption{Kernel distances between $\mu=\mathcal{N}(0,1)$ and $\nu=\mathcal{N}(5,1)$, as a function of the Gaussian kernel bandwidth $\sigma$.  
       \label{fig:varkernel}
}
\end{wrapfigure}From our Assumption~\ref{assum:kernel1} on the kernel, we have ensured that for any measure $\mu$, $\SL{1} \Sigma_\mu\SR{1} = 1$ which means that all measures representations considered are somehow ``normalised". On the contrary, for MMD with $k^2$ (or the Schatten 2-norm), $\SL{2} \Sigma_\mu\SR{2}$ the ``internal energy" depends on the measure (and on the kernel parameters such as bandwidth) and it can be smaller for distributions which are very flat, with high variance, as in general $k(x,y)\leq k(x,x)$ for $x\neq y$. This has consequences as intrinsically 
$\SL{2} \Sigma_\mu - \Sigma_\nu \SR{2} \leq \sqrt{\SL{2} \Sigma_\mu\SR{2}^2 + \SL{2} \Sigma_\nu\SR{2}^2}$
, the maximum value can be already small independently of the differences between $\mu$ and $\nu$. When minimizing  an objective such as $\mu\mapsto \SL{2} \Sigma_\mu - \Sigma_\nu \SR{2}$ (e.g., with gradient descent on the atoms in the support of $\mu$ if it is a discrete measure, as in \cite{arbel2019maximum}), this has an impact on the shape of the slope.
Moreover, the energy depends on the hyperparameters of the kernel, which are hard to tune for both the distributions' variances and the distance between their means at the same time. 

\Cref{fig:varkernel} illustrates this by displaying the two distances between sets of $n=1000$ samples from $\mathcal{N}(0,1)$ and $\mathcal{N}(5,1)$. We would expect sample sets to look closer as the Gaussian kernel bandwidth $\sigma$ grows, but for MMD that is not always the case. Other such phenomena are displayed by varying the variance or the mean of the distributions in the Appendix~\ref{appendssec:XPmmdpitfalls}.

Now let us consider 
two measures $\mu,\nu$ on $\mathcal{X}$ 
such that $\mathbb{E}_{X \sim \mu,Y \sim \nu} [k(X,Y)] \leq \epsilon$ for some small parameter $\epsilon>0$. 
Then, $\langle \Sigma_\mu ,\Sigma_\nu \rangle \leq \epsilon$ by Cauchy-Schwartz. 
Consider the density operator of the mixture $\Sigma_{\frac{1}{2}\mu+\frac{1}{2}\nu} = \frac{1}{2}\Sigma_{\mu}+ \frac{1}{2}\Sigma_{\nu}$, we have:
    \begin{align*}
     ||\Sigma_{\frac{1}{2}\mu+\frac{1}{2}\nu}||_1 = 1 = \frac{1}{2}||\Sigma_\mu||_1+\frac{1}{2}||\Sigma_\nu||_1, \qquad
   ||\Sigma_{\frac{1}{2}\mu+\frac{1}{2}\nu}||_2^2 \leq \frac{1}{2}\left(\frac{1}{2}||\Sigma_{\mu} ||_2^2 + \frac{1}{2}||\Sigma_{\nu} ||_2^2 +\epsilon\right).
    \end{align*}
We see that in contrast to the 1-Schatten norm, the 2-Schatten norm energy bound is roughly divided by 2 (as $\epsilon \to 0$, e.g. for almost orthogonals $\Sigma_{\mu},\Sigma_{\nu}$). 
Then, we reason with distance rather than norm:
\begin{proposition}\label{prop:modes}
    Let us consider distances between two mixtures $P=\frac{1}{2}\mu_1+\frac{1}{2}\mu_2$ and $Q=\frac{1}{2}\nu_1+\frac{1}{2}\nu_2$ such that $\Sigma_{\mu_1},\Sigma_{\nu_1}$ are orthogonal to $\Sigma_{\mu_2},\Sigma_{\nu_2}$. Then:
    \begin{align*}
        &\dist(P,Q) = \frac{1}{2}\dist(\mu_1,\nu_1) + \frac{1}{2}\dist(\mu_2,\nu_2)\\
        &\MMD_{k^2}^2(P,Q) = \frac{1}{4}\MMD_{k^2}^2(\mu_1,\nu_1) + \frac{1}{4}\MMD_{k^2}^2(\mu_2,\nu_2) .
    \end{align*}
    \normalsize
\end{proposition}
See proof in the Appendix~\ref{appendsssec:mmdpitfalls}.
If the distance between $\mu_2$ and $\nu_2$ are the same as between $\mu_1$ and $\nu_1$ (for instance, if the former are respective translation of the latter and the kernel is translation-invariant), we can see that the squared MMD distance loses a factor 2 while $\dist$ behaves similarly to the Total Variation of the mixtures when $\mu_1,\nu_1$ have different supports than $\mu_2,\nu_2$. 
This is the case when taking for instance in $\mathcal{X}=\mathbb{R}^2$ $\mu_1 = \mathcal{N}([0,0],I_2)$ and $\nu_1 = \mathcal{N}([0.3,0.3],I_2)$ while $\mu_2 = \mathcal{N}(\Delta,I_2)$ and $\nu_2 = \mathcal{N}(\Delta+[0.3,0.3],I_2)$ for $\Delta = [10,10]$.
In practice, the RKHS density operators are not perfectly orthogonal unless $||\Delta|| \to +\infty$ (in that case $\langle \Sigma_{\mu},\Sigma_{\nu}\rangle\to 0$ for a fixed bandwidth), but typically they can look so up to numerical precision, when using exponentially decreasing kernels (e.g., Gaussian). 
Taking $n=100$ samples each from each $\mu_1$ and $\nu_1$, and translating them by $\Delta$, the results above from Proposition~\ref{prop:modes} are confirmed numerically:
we find empirically $\widehat{\dist}(P,Q) = \widehat{\dist}(\mu_1,\nu_1) = 0.5992$ while $\widehat{\MMD_{k^2}^2}(\mu_1,\nu_1) = 0.0253$ but $\widehat{\MMD_{k^2}^2}(P,Q)=0.0127$, half of it (for a Gaussian kernel with bandwidth $\sigma=0.5$).

\subsection{Robustness}\label{subsec:robust}

We now turn to investigating the robustness 
of the kernel trace distance. 
In particular, we consider the $\epsilon$-contamination model, where
the training dataset is supposedly contaminated by a fraction $\epsilon \in (0, 1)$ of outliers \citep{huber1964robust}. The following proposition quantifies the robustness of this distance. 
    \begin{proposition}\label{prop:robust}
        Denote $P_\varepsilon = (1-\varepsilon) P + \varepsilon C$ where $C$ is some contamination distribution.
        We have when Assumption~\ref{assum:kernel1} is verified: $|\dist(P_\varepsilon,Q)-\dist(P,Q)| \leq 2 \varepsilon$.
    \end{proposition}
    The proof relies on the triangular inequality (see Appendix~\ref{appendsssec:robust}). 
    Hence, we see that $d_{KT}$ is robust 
    while for the Wasserstein distance,
    a contamination $C$ 
    arbitrarily ``far away from the distribution $Q$" will incur an arbitrarily high distance. The proof of robustness also works for MMD.

\section{Statistical Properties}
\label{sec:stat}
\subsection{Convergence rate}\label{subsec:rate}
In this section, we consider a measure $\mu$ and its empirical counterpart $\mu_n$ for $n$ independent samples and study the rate of convergence of $\dist(\mu,\mu_n)$. We note $A \lesssim_{\mu^{\otimes n}} b$ where $A$ is r.v., when for any $\delta>0$, there exists $c_\delta<\infty$ such that $\mu^{\otimes n}(A\leq c_\delta b )\geq \delta$.
With the Schatten 1-norm, it is not enough to study only the concentration of one (the maximal) eigenvalue as for the operator norm ($p=\infty$), we need to handle an infinity of eigenvalues (when the RKHS is of infinite dimension), neither can we use the Cauchy-Schwarz trick as for the Hilbert norm ($p=2$). 
However, since the trace of our kernel density operators are bounded by 1, only a few of the eigenvalues will have a significant contribution. Therefore, assuming some decay rate on those eigenvalues, we can focus on the convergence of operators on a subspace of the top eigenvectors, using results from the Kernel PCA literature. We introduce the  population and empirical square loss associated with some projector $P$:
\begin{equation*}
    R(P) = \mathbb{E}_{X\sim \mu} || \phi(X) - P \phi(X)||_\mathcal{H}^2, \qquad 
    R_n(P) = \sum_{i=1}^n \frac{1}{n} || \phi(x_i) - P \phi(x_i)||_\mathcal{H}^2
\end{equation*}
where the $(x_i)_{i=1...n}$ are each drawn independently from $\mu$.
We first make the following assumption, as in~\citet{sterge2020gain}.
\begin{assumptionp}{2} \label{assum:eigendef}
    The eigenvalues $(\lambda_i)_{i\in I}$ of $\Sigma_\mu$ (resp. $(\hat{\lambda}_j)_{j\in J}$ of $\Sigma_{\mu_n}$) are positive, simple and w.l.o.g. arranged in decreasing order ($\lambda_1\geq\lambda_2\geq...$).
\end{assumptionp}
This allows us to
denote $P^l(\Sigma_\mu)$ the projector on the subspace of the $l$ eigenvectors associated with the $l$ highest eigenvalues $\lambda_1,...,\lambda_l$.
Note that $\SL{1}P^l(\Sigma_\mu)\Sigma_\mu-\Sigma_\mu\SR{1} = \sum_{i>l} \lambda_l = R(P^l(\Sigma_\mu))$
(see for instance~\citet{blanchard2007statistical,rudi2013sample}). Similarly we consider $P^l(\Sigma_{\mu_n})$ for $\Sigma_{\mu_n}$. 


We now consider different kinds of assumptions on the decay rate of  eigenvalues of $\Sigma_\mu$ to get different corresponding  convergence rates, as in~\citet{sterge2020gain,sterge2022statistical}.
\begin{assumptionp}{P (Polynomial)} 
For some $\alpha>1$ and $0<\underline{A}<\bar{A}<\infty$, \begin{equation*}
        \underline{A}i^{-\alpha}\leq \lambda_i \leq \bar{A}i^{-\alpha}. \tag{P} \label{AssumPoly} 
    \end{equation*}
\end{assumptionp}

\begin{assumptionp}{E (Exponential)}
For $\tau>0$ and $\underline{B},\bar{B}\in(0,\infty)$,
            \begin{equation*}
                \underline{B} e^{-\tau i} \leq \lambda_i \leq \bar{B} e^{-\tau i}. \tag{E} \label{AssumExp}
            \end{equation*}
\end{assumptionp}

    
            

\begin{lemma}
\label{lem:decay}
    Suppose Assumption~\ref{assum:kernel1} and \ref{assum:eigendef} are verified.
    With a polynomial decay rate of order $\alpha>1$ (Assumption \ref{AssumPoly}), for $l=n^\frac{\theta}{\alpha},0<\theta\leq\alpha$: 
    \begin{equation}\label{eq:poly1}     \SL{1}P^l(\Sigma_\mu)\Sigma_\mu-\Sigma_\mu\SR{1} = R(P^l(\Sigma_\mu)) = \Theta \left( n^{-\theta(1-\frac{1}{\alpha})}\right),
\quad        \SL{2} P^l(\Sigma_{\mu})\Sigma_\mu -\Sigma_\mu \SR{2} =\Theta \left( n^{-\theta(1-\frac{1}{2\alpha})}\right),
    \end{equation}
    and there exists $N\in\mathbb{N}$ such that for $n>N$:
    \begin{equation}\label{eq:poly2emp}
    \hspace{-0.15cm}   \SL{2} P^l(\Sigma_{\mu_n})\Sigma_\mu -\Sigma_\mu \SR{2} \lesssim_{\mu^{\otimes n}} max(
        n^{-\frac{1}{2}+\frac{1}{4\alpha}},n^{-\theta+\frac{1}{4\alpha}})
        .
    \end{equation}
    With an exponential decay rate (Assumption \ref{AssumExp}), for $l=\frac{1}{\tau} \log n^\theta,\theta>0$:
\begin{equation}\label{eq:exp1} \SL{1}P^l(\Sigma_\mu)\Sigma_\mu-\Sigma_\mu\SR{1} =
        R(P^l(\Sigma_\mu)) = \Theta(n^{-\theta}),
  \qquad      \SL{2} P^l(\Sigma_{\mu})\Sigma_\mu -\Sigma_\mu \SR{2} = \Theta\left( n^{-\theta} \right)
    \end{equation}
    
    and there exists $N\in\mathbb{N}$ such that for $n>N$:
    \begin{equation}\label{eq:exp2emp}
            \SL{2} P^l(\Sigma_{\mu_n})\Sigma_\mu -\Sigma_\mu \SR{2} \lesssim_{\mu^{\otimes n}} \left\{
                \begin{array}{ll}
                    \sqrt{\frac{\log n}{n^\theta}} & \mbox{if } \theta < 1 \\
                     \frac{(\log n)}{\sqrt{n}} & \mbox{if } \theta \geq 1. 
                \end{array}
            \right.
    \end{equation}
    
\end{lemma}

The previous lemma (see proof in Appendix~\ref{appendsssec:rate}) is crucial to prove our main theorem below, that provides dimension-independent statistical rates.

\begin{theorem}\label{thm:rates}  Suppose Assumption \ref{assum:kernel1} and \ref{assum:eigendef} are verified.
\begin{itemize}  
    \item     If the eigenvalues of $\Sigma_\mu$ follow a polynomial decay rate of order $\alpha>1$ (Assumption~\ref{AssumPoly}), then: 
    \begin{equation*}
        \dist(\mu,\mu_n) \lesssim_{\mu^{\otimes n}} 
        n^{-\frac{1}{2}+\frac{1}{2\alpha}}.
    \end{equation*}
    \item   If the eigenvalues of $\Sigma_\mu$ follow an exponential decay rate (Assumption~\ref{AssumExp}), then:
    \begin{equation*}
        \dist(\mu,\mu_n) \lesssim_{\mu^{\otimes n}} 
        \frac{(\log n)^\frac{3}{2}}{\sqrt{n}}.
    \end{equation*}
\end{itemize}
\end{theorem}
\vspace{-0.4cm}

\begin{proofsketch}
For clarity of notation, we abbreviate $\Sigma_\mu$ and $\Sigma_{\mu_n}$ as $\Sigma$ and $\Sigma_n$.
By the triangular inequality:
\begin{align} 
    ||\Sigma-\Sigma_n||_1 \leq 
    ||\Sigma - P^l(\Sigma)\Sigma ||_1 
    + ||(P^l(\Sigma) -P^l(\Sigma_n))\Sigma ||_1 \notag 
    + ||P^l(\Sigma_n)(\Sigma -\Sigma_n) ||_1
\\
\label{ineq:triangleineq}
   +||P^l(\Sigma_n)\Sigma_n - \Sigma_n ||_1 \coloneqq (A) + (B) + (C) + (D)
\end{align}
We bound each term of eq.~\ref{ineq:triangleineq}.
Term (A) is bounded using Lemma~\ref{lem:decay}. Similarly, (D) relates to (A) by a result due to ~\citet{blanchard2007statistical} (eq. (30)), see Lemma~\ref{lem:blanchard} in Appendix~\ref{appendsssec:rate}. For (B) and (C), the projections allow to work in a subspace of dimension at most $2l$ and by eq.~(3) (Hölder's inequality) to relate to the Schatten 2-norm which has rates like MMD. Finally, we pick $\theta=\frac{1}{2}$ for polynomial decay and $\theta=1$ for the exponential decay (see Lemma~\ref{lem:decay}) to minimise the maximum of the four terms.
See Appendix~\ref{appendsssec:rate} for the full proof.
\end{proofsketch}

By the Fuchs-van de Graaf inequality (Eq.~(\ref{ineq:FvdG}) and (\ref{ineq:d1kbw})), it directly implies (also dimensionally-independent) convergence rates for the Kernel Bures Wasserstein distance, that are novel to the best of our knowledge.
\begin{corollary}
    Suppose Assumption \ref{assum:kernel1} and \ref{assum:eigendef} verified.
    If Assumption~\ref{AssumPoly} is verified:
      $  \dKBW(\mu,\mu_n) \lesssim_{\mu^{\otimes n}} n^{-\frac{1}{4}+\frac{1}{4\alpha}}.$
     If Assumption~\ref{AssumExp} is verified:
       $ \dKBW(\mu,\mu_n) \lesssim_{\mu^{\otimes n}} (\log n)^\frac{3}{4} n^{-\frac{1}{4}}.$
\end{corollary}

\section{Experiments}\label{sec:exp}
In this section, we illustrate the interest of our novel kernel trace distance on different experiments.
\vspace{-0.2cm}
\paragraph{Approximate Bayesian Computation (ABC)}
The purpose of Approximate Bayesian Computation~\citep{tavare1997inferring} is to compute an approximation of the posterior when doing Bayesian inference 
in a likelihood-free fashion.
The idea of using a distance $d$ between distributions
 to build a synthetic likelihood has recently flourished~\citep{frazier2020robust,
bernton2019approximate,
jiang2018approximate}. 
ABC methods based on IPM 
enjoy theoretical guarantees~\citep{legramanti2022concentration}. The ABC posterior distribution is defined by
$\pi(\theta|X^{n}) 
    \propto \int \pi(\theta) \mathds{1}_{\{d(X^{n},Y^{m}) < \epsilon\}}p_{\theta}(Y^{m}) \mathrm{d}Y^{m}$,
where $\pi(\theta)$ is a prior over the parameter space $\Theta$, $\epsilon > 0$ is a tolerance threshold, and $Y^m$ are synthetic data generated according to $p_{\theta}(Y^{m}) = \prod_{j=1}^{m}p_{\theta}(Y_{j})$. It is approximately computed by drawing $\theta_i \sim \pi$ for $i=1,...,T$ and simulating synthetic data $Y^{m}\sim p_{\theta_i}$ and keeping or rejecting $\theta_i$ according to whether the synthetic data is close to the real data. The result is a list $L_\theta$ of all accepted $\theta_i$ (see Algo.~\ref{alg:reject_abc} in the Appendix~\ref{appendssec:XPABC}).

Here, as we are interested in robustness, we will consider 
a contamination case 
using 
 Normal distributions but where nonetheless the usual likelihood fails 
to recover the correct mean 
 as the data is corrupted. We will take as prior $\pi = \mathcal{N}(0,\sigma_0^2)$ and the real data consist of $n=100$ samples coming following $\mu^*=\mathcal{N}(\theta^*=1,1)$ where $10\%$ of the samples are replaced by contaminations from $\mathcal{N}(20,1)$. We fit the model $p_\theta = \mathcal{N}(\theta,1)$ by picking the best $\theta$ possible. We carry out $T=10000$ iterations, generating each times $m=n$ synthetic data.

We consider ABC with the threshold value $\epsilon=0.05,0.25,0.5,1$. 
For the traditional likelihood, 
Bayes' rule gives posterior 
$p(\theta|x)=\mathcal{N}(\frac{\sum_{i=1}^n x_i}{n+\frac{1}{\sigma_0^2}},\frac{1}{n+\frac{1}{\sigma_0^2}})$. 
Since $\mathbb{E}[X_i]=0.9\times 1+0.1\times 20=2.9$ the location is therefore in expectation $\mathbb{E}[\frac{\sum_{i=1}^n x_i}{n+\frac{1}{\sigma_0^2}}]=\frac{n}{n+\frac{1}{\sigma_0^2}} 2.9  \approx 2.9$, the contamination significantly impacted the posterior.
Similarly, for any model $p_\theta$, the Wasserstein distance with the contaminated mixture $0.9 \mathcal{N}(1,1) + 0.1 \mathcal{N}(20,1)$ will be high, and empirically all of the $T$ iterations are rejected for  all the values of $\epsilon$ considered. Thus, we disregard the Wasserstein distance from the experiment and compare the performance of MMD to that of $\dist$. We also consider concurrent methods out of our scope such as 
MMD with the unbounded energy kernel: $k(x,y)=\frac{1}{2}(||x||+||y||-|| x-y||)$ ~\cite{sejdinovic2013equivalence}, and others displayed in Appendix~\ref{appendssec:XPABC}.

We measure the average Mean Square Error between the target parameter $\theta^*=1$ and the accepted $\theta_i\in L_\theta$: $\widehat{MSE}= \frac{1}{|L_\theta|}\sum_{\theta_i\in L_\theta} ||\theta_i-\theta^*||^2$ which also corresponds to the average of squared Wasserstein 2-distance as $W_2^2(\mu^*,p_{\theta_i})=||\theta_i-\theta^*||^2$ since we consider only Gaussians with same variance. We picked $\sigma_0=5$ for the prior. We repeat $10$ times the experiment with fresh samples, the averaged results 
are shown in Table~\ref{tab:abc}. 
As expected -- and discussed in subsection~\ref{subsec:mmdpitfalls} -- MMD (gaussian) is too lenient to accept. For $\epsilon=0.05$ inferior to the contamination level ($10\%$), it still accept $11\%$ of the times, while $\dist$ reject all the times, which can be understood as $\dist$ detecting the contamination, that prevents to match with the Gaussian model. 
The energy kernel can not help enough to beat $\dist$. The densities of the obtained posteriors  are shown alongside the target in Fig.~(\ref{fig:posteriors}) and (\ref{fig:otherposteriors}) in the Appendix~\ref{appendssec:XPABC}.

\begin{table}[h]
\caption{Average MSE of ABC Results. \\ 
The Gaussian kernel is used with $\sigma=1$ (as the variance of $p_\theta$ and $\mu^*$).
As expected, MMD is too lenient to accept most sampled $\theta_i$ leading to a high average MSE unless $\varepsilon$ is carefully chosen. Whereas the proposed $\dist$ discriminates between the correct and the wrong $\theta_i$ for $\varepsilon$ larger than the contamination threshold $0.1$. MMD is assumed to use the Gaussian kernel while 
$\mathrm{MMD_E}$ denotes the 
MMD with the energy kernel.
\label{tab:abc} }
\centering
\begin{tabular}{c ccc ccc ccc}
\toprule
 $\varepsilon$ & \multicolumn{3}{c}{0.05} & \multicolumn{3}{c}{0.25} & \multicolumn{3}{c}{0.5} \\
 \midrule
 distance & $\mathrm{MMD}$ & $\mathrm{MMD_E}$ &  $\dist$ 
 & $\mathrm{MMD}$ & $\mathrm{MMD_E}$ &  $\dist$ 
 & $\mathrm{MMD}$ & $\mathrm{MMD_E}$ &  $\dist$ \\
\cmidrule(lr){2-4} \cmidrule(lr){5-7} \cmidrule(lr){8-10}%
 \#accept. & 1092 & 0 & 0 & 2964 & 0 & 58 & 6168 & 846 & 828\\
 MSE &  0.19 & N/A & N/A & 1.29 & N/A & \textbf{0.03} & 7.47 & 0.17 & 0.12\\  
\bottomrule
\end{tabular}
\end{table}

\paragraph{Particle Flow}
We consider
the performance of gradient descent 
when optimizing  $\mu \mapsto d_{KT}(\mu,\nu)$ for discrete measures $\mu,\nu$ on $\mathbb{R}^2$, given an  initial point cloud (in red) and a target cloud of points (in blue) both of $n=100$ points. We run the scheme with a learning rate of $0.005$ for $1000$ steps, using $\dist$ (Schatten 1-norm) 
and MMD (Schatten 2-norm), 
see Appendix~\ref{appendssec:XPGF} (Figs.~\ref{fig:flowd1} and \ref{fig:flowmmd}).
We use the Laplacian kernel: $k(x,y)=e^{-\frac{||x-y||_1}{\sigma}}$ where here $||\cdot||_1$ 
means the $l_1$ norm for vectors.
We choose a bandwidth $\sigma=1$ (as the image size is a unit square) for $\dist$ and for MMD we use $k^2$ as kernel to match the Schatten 2-norm (i.e. we use $\sigma=0.5$ instead of $\sigma=1$, and it gives a better convergence). 
The inherent internal energy of MMD incites the point cloud to spread out and therefore some particles are still left out far away from the target, which does not happen with $\dist$.



\section{Conclusion}
We introduced a 
robust distance between probability measures, based on 
RKHS density (or covariance) 
operators 
and their Schatten-1 norm.
It is the greatest in a family of kernel-based IPM including 
MMD, and so is more discriminative as shown in experiments. 
We show how to compute it between discrete measures via a new kernel trick.
Assuming some decay rate of the eigenvalues of the RKHS density operator leads to a statistical convergence rate 
that can be  close to $O(n^{-\frac{1}{2}})$.  This implies 
the first (dimension-independent) rates 
for the Kernel Bures Wasserstein distance. 
Future work includes reducing computational complexity via Nyström method
, improving the dependence on the order of decay $\alpha$, as well as minimax lower bounds.
\bibliography{ref}
\bibliographystyle{plainnat}


\newpage
\appendix
\onecolumn

\section{PROOFS}\label{appendsec:proofs}

\subsection{Discriminative properties (proofs of section~\ref{sec:motivation})}
\subsubsection{Comparison with other distances (proofs of subsection~\ref{subsec:firstprop})} \label{appendsssec:firstprop}

\newtheorem*{prop:MMDschatten2}{Lemma \ref{prop:MMDschatten2}}
\begin{prop:MMDschatten2}
The Schatten 2-norm of the difference of the kernel density operators (using kernel $k$) of two distributions corresponds to their Maximum Mean Discrepancy using the kernel $k^2$:
\begin{equation}
     \STL\Sigma_\mu-\Sigma_\nu\STR = MMD_{k^2}(\mu,\nu)   
\end{equation}
Consequently, $MMD_{k^2}(\mu,\nu) \leq \dist(\mu,\nu)$
\end{prop:MMDschatten2}

\begin{proof}
    We have :
    \begin{align*}
        \langle \Sigma_\mu, \Sigma_\nu\rangle &= Tr(\Sigma_\mu \Sigma_\nu^*) 
        = Tr(\Sigma_\mu \Sigma_\nu) \\
        &= Tr(\int_\mathcal{X} \varphi(x)\varphi(x)^* \mu(x)dx  \int_\mathcal{Y} \varphi(y)\varphi(y)^* \nu(y)dy) \\
        &= \int_\mathcal{X} \int_\mathcal{Y} Tr(\varphi(x)\varphi(x)^*  \varphi(y)\varphi(y)^*) \mu(x) \nu(y)dxdy \\
        &= \int_\mathcal{X} \int_\mathcal{Y} k(x,y)k(x,y) \mu(x)  \nu(y)dx dy
    \end{align*}
    And $\SL{2} \Sigma_\mu-\Sigma_\nu \SR{2}^2 = \langle \Sigma_\mu, \Sigma_\mu \rangle + \langle \Sigma_\nu, \Sigma_\nu\rangle  -2\langle \Sigma_\mu, \Sigma_\nu\rangle$, hence the result.
\qed
\end{proof}

\begin{lemma}\label{lem:gauss}
For the Gaussian kernel $k(x,y) = e^{-\frac{||x-y||^2}{2\sigma^2}}$, we have:
\begin{align*}
    ||\varphi(x) - \varphi(y) ||_\mathcal{H} \leq \frac{||x-y||}{\sigma}
\end{align*}
\end{lemma}

\begin{proof}
    \begin{align*}
    ||\varphi(x) - \varphi(y) ||_\mathcal{H} &= \sqrt{||\varphi(x)||_\mathcal{H}^2+||\varphi(y)||_\mathcal{H}^2-2k(x,y)}\\
    &=\sqrt{2(1-k(x,y))} \\
    &\leq \sqrt{2 \frac{||x-y||^2}{2\sigma^2}}\\
    &= \frac{||x-y||}{\sigma}
\end{align*}
where we used $1-e^{-x}\leq x$.
\qed
\end{proof}

Thanks to Lemma~\ref{lem:gauss}, we can prove Corollary~\ref{cor:wasserineq}:

\newtheorem*{cor:wasserineq}{Corollary~\ref{cor:wasserineq}}
\begin{cor:wasserineq}
If Assumption~\ref{assum:kernel1} is verified,
\begin{align*}
    \dist(\mu,\nu) \leq 2 W_{c_k}(\mu,\nu).
\end{align*}
Furthermore, using the Gaussian kernel with parameter $\sigma$,
\begin{align*}
    \dist(\mu,\nu) \leq 2 W_{c_k}(\mu,\nu)\leq \frac{2}{\sigma} W_{||.||}(\mu,\nu). 
\end{align*}
\end{cor:wasserineq}

\begin{proof}
    For any coupling $\pi$ of $\mu$ and $\nu$, any $f\in\mathcal{F}_1$, thanks to Proposition~\ref{thm:IPM}:
    \begin{align*}
        \mathbb{E}_\mu[f(X)] - \mathbb{E}_\nu[f(Y)] &=\mathbb{E}_\pi [f(X) - f(Y)]\\
        &\leq 2 \mathbb{E}_\pi ||\varphi(X)-\varphi(Y)||_\mathcal{H}
    \end{align*}
    This is true for any coupling and any function of $\mathcal{F}_1$ so it stays true when minimising over all couplings and maximising over all $f\in\mathcal{F}_1$. The second inequality follows similarly from Lemma~\ref{lem:gauss}.
    
\end{proof}

Finally, we show how to use the \citet{FuchsDeGraaf} inequalities to frame the kernel trace distance with the Kernel Bures-Wasserstein distance:
\begin{lemma}\label{lem:d1kbw}
When Assumption~\ref{assum:kernel1} is verified,
\begin{align}
     d_{KBW}(\mu,\nu)^2 \leq \dist(\mu,\nu) 
    \leq 2 d_{KBW}(A,B)
\end{align}    
\end{lemma}

\begin{proof}

We have $F(\Sigma_\mu,\Sigma_\nu)=1-\frac{1}{2}d_{KBW}(\mu,\nu)^2$.
Therefore, from $2(1 - F(\Sigma_\mu,\Sigma_\nu) )\leq ||\Sigma_\mu-\Sigma_\nu||_1$, we get $d_{KBW}(\mu,\nu)^2 \leq \dist(\mu,\nu) $. For the second inequality:
\begin{align*}
    ||\Sigma_\mu-\Sigma_\nu||_1 &\leq 2\sqrt{1 - F(\Sigma_\mu,\Sigma_\nu)^2}\\
            &= 2\sqrt{1 - (1-\frac{1}{2}d_{KBW}(\mu,\nu)^2)^2}\\
            & = 2\sqrt{1 - (1-\frac{1}{4}d_{KBW}(\mu,\nu)^4-d_{KBW}(\mu,\nu)^2)}\\
            &=2\sqrt{d_{KBW}(\mu,\nu)^2-\frac{1}{4}d_{KBW}(\mu,\nu)^4}\\
            &=2d_{KBW}(\mu,\nu)\sqrt{1-\frac{1}{4}d_{KBW}(\mu,\nu)^2}\\
            &\leq 2d_{KBW}(\mu,\nu)
\end{align*}
\qed
\end{proof}

\subsection{Proof of \Cref{thm:IPM}}\label{sec:proof_IPM}
\begin{proof}
The definition as an IPM \ref{item:1} comes from the dual definition in Eq.~(\ref{eq:schattendual}).  Indeed,
\begin{align*}
    || \Sigma_\mu - \Sigma_\nu ||_1 &= \sup_{U \in \mathcal{L(H)}, ||U||_\infty = 1} \langle U,\Sigma_\mu - \Sigma_\nu \rangle.
\end{align*}
Then,  $ \langle U,\Sigma_\mu\rangle = \Tr(U^*\mathbb{E}_{X\sim \mu}[\varphi(X)\varphi(X)^*]) = \mathbb{E}_{X\sim \mu}[\Tr(U^*\varphi(X)\varphi(X)^*)]$ and similarly for $\Sigma_\nu$.
For symmetric (around zero) space of functions, one can drop the absolute values in the definition of IPM.
So the formulation of $\mathcal{F}_1$ comes from  $\Tr(U^*\varphi(x)\varphi(x)^*) = \varphi(x)^*U^*\varphi(x)$ (equivalently, picking $U^*$ instead of $U$, since adjunction applied to $\{U \in \mathcal{L(H)}, ||U||_\infty = 1\}$ is a bijection from this space to itself, 
gives $\varphi(x)^*U\varphi(x)$).

Regarding \ref{item:2}, for $x \in \mathcal{X}, \, U\in\mathcal{L(H)}$ with $||U||_\infty = 1$:
    \begin{align*}
        | \varphi(x)^*U\varphi(x) | &\leq ||\varphi(x)||_\mathcal{H} ||U\varphi(x)||_\mathcal{H} \\
        &\leq ||\varphi(x)||_\mathcal{H} ||\varphi(x)||_\mathcal{H} = 1,
    \end{align*}
    where the equality comes from \Cref{assum:kernel1}.
    
Then concerning \ref{item:3}, for $x,y\in \mathcal{X}$:
\begin{align*}
|f(x)-f(y)| &= |\varphi(x)^*U\varphi(x) - \varphi(y)^*U\varphi(y)|\\
&= |\varphi(x)^*U\varphi(x) - \varphi(y)^*U\varphi(x) + \varphi(y)^*U\varphi(x) - \varphi(y)^*U\varphi(y)|\\
&= |(\varphi(x)-\varphi(y))^*U\varphi(x) - \varphi(y)^*U(\varphi(y)-\varphi(x))|\\
&\leq |(\varphi(x)-\varphi(y))^*U\varphi(x)| + |\varphi(y)^*U(\varphi(y)-\varphi(x))|\\
&\leq ||(\varphi(x)-\varphi(y))|| \cdot ||U\varphi(x)|| + ||\varphi(y)||\cdot ||U(\varphi(y)-\varphi(x))||
\end{align*}

by Cauchy-Schwartz.
Because $||U||_\infty \leq 1 $:
\begin{align*}
|f(x)-f(y)| \leq ||\varphi(x)-\varphi(y)|| \cdot ||\varphi(x)|| + ||\varphi(y)|| \cdot||\varphi(y)-\varphi(x)||
\end{align*}
therefore by Assumption 1:
\begin{align*}
|f(x)-f(y)| \leq 2||\varphi(x)-\varphi(y)||
\end{align*}
\qed
\end{proof}

\subsubsection{Normalised energy (proof of subsection~\ref{subsec:mmdpitfalls})}\label{appendsssec:mmdpitfalls}
\newtheorem*{prop:modes}{Proposition \ref{prop:modes}}
\begin{prop:modes}
    Let's consider distances between two mixtures $P=\frac{1}{2}\mu_1+\frac{1}{2}\mu_2$ and $Q=\frac{1}{2}\nu_1+\frac{1}{2}\nu_2$ such that $\Sigma_{\mu_1},\Sigma_{\nu_1}$ are orthogonal to $\Sigma_{\mu_2},\Sigma_{\nu_2}$. Then:
    \begin{align*}
        &\dist(P,Q) = \frac{1}{2}\dist(\mu_1,\nu_1) + \frac{1}{2}\dist(\mu_2,\nu_2)\\
        &MMD_{k^2}^2(P,Q) = \frac{1}{4}MMD_{k^2}^2(\mu_1,\nu_1) + \frac{1}{4}MMD_{k^2}^2(\mu_2,\nu_2) .
    \end{align*}
    \normalsize
\end{prop:modes}

\begin{proof}
 Noting $\alpha = \mu_1 - \nu_1$ and $\beta = \mu_2-\nu_2$ , notice that $\Sigma_P-\Sigma_Q=\Sigma_\alpha+\Sigma_\beta$, and that $\Sigma_\alpha \perp \Sigma_\beta$.
Looking at the dual expression~(\ref{eq:schattendual}) of the Schatten norm:
\begin{align}
    &\SOL \frac{1}{2}\Sigma_\alpha + \frac{1}{2}\Sigma_\beta  \SOR = \sup_{U \in \mathcal{L(H)}, ||U||_\infty = 1} \langle U,\frac{1}{2}\Sigma_\alpha + \frac{1}{2}\Sigma_\beta  \rangle \\
                                        &= \sup_{U=U_1+U_2 \in \mathcal{L(H)},
                                        U_1\in span(\Sigma_\alpha),
                                        U_2\in span(\Sigma_\beta),
                                        ||U_1+U_2||_\infty = 1} \frac{1}{2}\langle U_1,\Sigma_\alpha  \rangle + \frac{1}{2}\langle U_2,\Sigma_\beta  \rangle
\end{align}
where by orthogonality we decomposed without loss of generality $U=U_1+U_2$ where $U_1$ and $U_2$ are restricted to the subspaces defined respectively by $\Sigma_\alpha$ and $\Sigma_\beta$ and so mutually orthogonal.
Therefore $||U_1+U_2||\infty = \max{(||U_1||_\infty,||U_2||_\infty)}$ by orthogonality, we can maximise using $||U_1||_\infty=||U_2||_\infty=1$ and recover $||\frac{1}{2}\Sigma_\alpha + \frac{1}{2}\Sigma_\beta||_1 = \frac{1}{2}||\Sigma_\alpha||_1+\frac{1}{2}||\Sigma_\beta||_1$. However, for $p=2$, 
we get by 
the definition of the Schatten norm 
and orthogonality, $||\frac{1}{2}\Sigma_\alpha + \frac{1}{2}\Sigma_\beta||_2^2 = ||\frac{1}{2}\Sigma_\alpha ||_2^2 + || \frac{1}{2}\Sigma_\beta||_2^2 =  \frac{1}{4}||\Sigma_\alpha ||_2^2 + \frac{1}{4}|| \Sigma_\beta||_2^2$.
\qed
\end{proof}

In particular, when $||\Sigma_\alpha||_2=||\Sigma_\beta||_2$ (for instance, $\beta$ is a translation of $\alpha$ and the kernel is translation-invariant), $||\frac{1}{2}\Sigma_\alpha + \frac{1}{2}\Sigma_\beta||_2^2 = \frac{1}{2}||\Sigma_\alpha ||_2^2$, there is a decrease by a factor $\frac{1}{2}$.

\subsection{Statistical properties (proofs of section~\ref{sec:stat})}

\subsubsection{Convergence rate (proof of subsection~\ref{subsec:rate})}
\label{appendsssec:rate}

Remember that for clarity of notation, in the proofs we may abbreviate $\Sigma_\mu$ and $\Sigma_{\mu_n}$ as $\Sigma$ and $\Sigma_n$.

\newtheorem*{lem:decay}{Lemma~\ref{lem:decay}}
\begin{lem:decay} 
    Suppose Assumption~\ref{assum:kernel1} and \ref{assum:eigendef} are verified.
    With a polynomial decay rate of order $\alpha>1$ (Assumption \ref{AssumPoly}), for $l=n^\frac{\theta}{\alpha},0<\theta\leq\alpha$: 
    \begin{equation*}
        \hspace{-0.22cm}\SL{1}P^l(\Sigma_\mu)\Sigma_\mu-\Sigma_\mu\SR{1} = R(P^l(\Sigma_\mu)) = \Theta \left( n^{-\theta(1-\frac{1}{\alpha})}\right),
    \end{equation*}
    \begin{equation*}
        \SL{2} P^l(\Sigma_{\mu})\Sigma_\mu -\Sigma_\mu \SR{2} =\Theta \left( n^{-\theta(1-\frac{1}{2\alpha})}\right),
    \end{equation*}
    and there exists $N\in\mathbb{N}$ such that for $n>N$:
    \begin{equation*}
    \hspace{-0.15cm}   \SL{2} P^l(\Sigma_{\mu_n})\Sigma_\mu -\Sigma_\mu \SR{2} \lesssim_{\mu^{\otimes n}} max(
        n^{-\frac{1}{2}+\frac{1}{4\alpha}},n^{-\theta+\frac{1}{4\alpha}})
        .
    \end{equation*}
    With an exponential decay rate (Assumption \ref{AssumExp}), for $l=\frac{1}{\tau} \log n^\theta,\theta>0$:
    \begin{equation*}
        \SL{1}P^l(\Sigma_\mu)\Sigma_\mu-\Sigma_\mu\SR{1} =
        R(P^l(\Sigma_\mu)) = \Theta(n^{-\theta}),
    \end{equation*}
    \begin{equation*}
        \SL{2} P^l(\Sigma_{\mu})\Sigma_\mu -\Sigma_\mu \SR{2} = \Theta\left( n^{-\theta} \right)
    \end{equation*}
    
    and there exists $N\in\mathbb{N}$ such that for $n>N$:
    \begin{equation*}
            \SL{2} P^l(\Sigma_{\mu_n})\Sigma_\mu -\Sigma_\mu \SR{2} \lesssim_{\mu^{\otimes n}} \left\{
                \begin{array}{ll}
                    \sqrt{\frac{\log n}{n^\theta}} & \mbox{if } \theta < 1 \\
                     \frac{(\log n)}{\sqrt{n}} & \mbox{if } \theta \geq 1. 
                \end{array}
            \right.
    \end{equation*}
    \end{lem:decay}

\begin{proof}
    The proof of the first point concerning $R(P^l(\Sigma_\mu))$ for both polynomial and exponential decay can be found for instance in Corollary 3 and 4 of~\citet{sterge2020gain} which is just bracketing $R(P^l(\Sigma_\mu)) = \sum_{i>l} \lambda_i = \Theta(\sum_{i>l} f(i))$ by some integrals of $f$ (which is the function of polynomial or exponential decay).

    The proof of the second point concerning $\SL{2} P^l(\Sigma_{\mu})\Sigma_\mu -\Sigma_\mu \SR{2} = \sqrt{\sum_{i>l} \lambda_i^2}$ is very similar. For the polynomial decay, the $(\lambda_i)^2$ verify the polynomial decay for $\alpha'=2\alpha$, and $\theta'=2\theta$ so that $0<\theta'\leq\alpha'$ is equivalent to $0<\theta\leq\alpha$ and $n^\frac{\theta'}{\alpha'}=n^\frac{\theta}{\alpha}=l$. Then taking the square root gives the result. Similarly for the exponential decay, the $(\lambda_i)^2$ verify the exponential decay for $\tau'=2\tau$ and $\theta'=2\theta$.

    Now for the proof of the third point, denote $\Sigma_t = \Sigma + t Id$ and similarly for $\Sigma_{n,t} = \Sigma_n + t Id$. Most previous works~\citep{sriperumbudur2022approximate,sterge2020gain} use what they call $\mathcal{N}_\Sigma(t)=Tr(\Sigma (\Sigma + t Id)^{-1})=|| \Sigma^\frac{1}{2}\Sigma_t^{-\frac{1}{2}}||^2_2$, but for us, we will consider rather $|| \Sigma \Sigma_t^{-1}||_2$ and use a result by~\citet{rudi2013sample}.
    The proof in the case of the polynomial decay rate concerning $\SL{2} P^l(\Sigma_{\mu_n})\Sigma_\mu -\Sigma_\mu \SR{2}$ or for short $|| (I-P^l(\Sigma_n))\Sigma||_2$ goes as the following:
    \begin{align}
        || (I-P^l(\Sigma_n))\Sigma||_2 &= ||(I-P^l(\Sigma_n))\Sigma_{n,t}\Sigma_{n,t}^{-1}\Sigma_t\Sigma_t^{-1}\Sigma||_2 \\
        & \leq ||(I-P^l(\Sigma_n))\Sigma_{n,t}||_\infty ||\Sigma_{n,t}^{-1}\Sigma_t||_\infty ||\Sigma_t^{-1}\Sigma||_2.\label{eq:tripleineq}
    \end{align}
    We have:
    \begin{align}\label{eq:I-P}
    ||(I-P^l(\Sigma_n))\Sigma_{n,t}||_\infty &= \widehat{\lambda}_{l+1} + t    \notag \\
    &\leq \frac{3}{2}(\lambda_l+t)
    \end{align}
    with probability $1-\delta$ for $\frac{\kappa}{n}\log(\frac{n}{\delta})\leq t\leq \lambda_1$ for some constant $\kappa$ according to Lemma A.1 (iii) of~\citet{sterge2020gain}.
    Applying Lemma 7.3 of~\citet{rudi2013sample}, we have thanks to the polynomial decay: 
    \begin{equation}\label{eq:Rudi}
        ||\Sigma_t^{-1}\Sigma||_2 = O(t^{-\frac{1}{2\alpha}}).
    \end{equation}
    Finally, let us show that $||\Sigma_{n,t}^{-1}\Sigma_t||_\infty$ is bounded with high probability for an appropriate range of $t$.
    If we note $B_n = \Sigma_{n,t}^{-1}(\Sigma_n-\Sigma)$, then $\Sigma_{n,t}^{-1}\Sigma_t = Id - B_n$. 
    Therefore bounding $||B_n||_\infty$ w.h.p. we can conclude by $||\Sigma_{n,t}^{-1}\Sigma_t||_\infty \leq 1 + ||B_n||_\infty$.
    So to bound $||B_n||_\infty$, notice that: 
    \begin{align}
        || B_n ||_\infty &\leq || \Sigma_{n,t}^{-1}||_\infty ||\Sigma_n-\Sigma||_\infty \\ 
        &= \frac{||\Sigma_n-\Sigma||_\infty}{t}. \label{eq:Bn}
    \end{align}
    
    Next, we want to apply a concentration bound to $||\Sigma_n-\Sigma||_\infty$ to see what range of $t$ can be handled.
    If we write $X_k = \frac{1}{n} (\varphi(x_k)\varphi(x_k)^*-\Sigma)$ so that $\Sigma_n-\Sigma = \sum_{k=1}^n X_k$, then we have $\mathbb{E}[X_k]=0$ and $||X_k||_\infty\leq \frac{2}{n}$. Besides 
    \begin{equation*}
           X_k^2 = (\frac{1}{n})^2(\varphi(x_k)\varphi(x_k)^*-\varphi(x_k)\varphi(x_k)^*\Sigma-\Sigma \varphi(x_k)\varphi(x_k)^* + \Sigma^2) 
    \end{equation*}
    since $\varphi(x_k)\varphi(x_k)^*\varphi(x_k)\varphi(x_k)^*=\varphi(x_k)\varphi(x_k)^*$, therefore 
    \begin{equation*}
            \textbf{Var}[\Sigma_n-\Sigma]=\sum_{k=1}^n X_k^2 \\
            = \frac{1}{n}(\Sigma -\Sigma^2)
            \preceq \frac{1}{n}\Sigma 
    \end{equation*}
    because of the positivity of $\Sigma^2$. With all this, we are ready to apply the 
    Matrix Bernstein inequality for the Hermitian case with intrinsic dimension (Theorem 7.7.1 of~\citet{tropp2015introduction} with $L=\frac{2}{n}$, $V=\frac{1}{n}\Sigma$, $d=intdim(V)=\frac{1}{||\Sigma||_\infty},v=\frac{||\Sigma||_\infty}{n})$ and get for $t'\geq \sqrt{v}+L/3$:
    \begin{equation}
        \mathbb{P}(||\Sigma_n-\Sigma||_\infty \geq t') \leq \frac{1}{||\Sigma||_\infty} \exp\left(\frac{-t'^2/2}{||\Sigma||_\infty/n+2t'/3n}\right)
    \end{equation}
    So assuming we can pick $t'=t$ for instance and combining with eq.~(\ref{eq:Bn}):
    \begin{equation}\label{eq:BnConcentration}
        \mathbb{P}(||B_n||_\infty \leq 1) \geq 1-\frac{1}{||\Sigma||_\infty} \exp\left(\frac{-nt^2/2}{||\Sigma||_\infty+2t/3}\right)
    \end{equation}
    Let us pick $t=\frac{K}{\sqrt{n}}$ for some $K$ big enough, so that the condition on $t'$ and the condition for eq~(\ref{eq:I-P}) to work are satisfied and the exponential term in eq.~(\ref{eq:BnConcentration}) is as small as desired to make the bound sensible (it is true for $n$ big enough).
    Combining the latter eq.~(\ref{eq:BnConcentration}) with eq.~(\ref{eq:I-P}) and (\ref{eq:Rudi}), eq.~(\ref{eq:tripleineq}) gives
    \begin{equation*}
        || (I-P^l(\Sigma_n))\Sigma||_2 \lesssim_{\mu^{\otimes n}}  t^{-\frac{1}{2\alpha}}(\lambda_{l+1} + t)
    \end{equation*}
    and replacing $t=\frac{K}{\sqrt{n}}$ and $\lambda_{l+1} = \Theta((l+1)^{-\alpha})=\Theta(n^{-\theta})$:
    \begin{equation*}
        || (I-P^l(\Sigma_n))\Sigma||_2 \lesssim_{\mu^{\otimes n}}  n^{-\frac{1}{2}+\frac{1}{4\alpha}} + n^{-\theta+\frac{1}{4\alpha}}
    \end{equation*}
    hence the result.

    Now for the case of exponential decay rate, since it is a more powerful hypothesis we use a simpler argument: 
    \begin{align*}
        || (I-P^l(\Sigma_n))\Sigma||_2 &\leq || (I-P^l(\Sigma_n))\Sigma^{1/2}||_2 ||\Sigma^{1/2}||_\infty\\
        &\leq || (I-P^l(\Sigma_n))\Sigma^{1/2}||_2
    \end{align*}
    since $||\Sigma||_1=1$, and it turns out that $||(I-P^l(\Sigma_n))\Sigma^{1/2}||_2=\sqrt{R(P^l(\Sigma_n))}$.
    According to~\citet{sterge2020gain},
    \begin{equation}
        R(P^l(\Sigma_n)) \lesssim_{\mu^{\otimes n}} \left\{
                \begin{array}{ll}
                    \frac{\log n}{n^\theta} & \mbox{if } \theta < 1 \\
                     \frac{(\log n)^2}{n} & \mbox{if } \theta \geq 1 
                \end{array}
            \right.
    \end{equation}
    hence the result.
    \qed
\end{proof}


By the work of~\citet{blanchard2007statistical}, we can state one of their results in a simplified lemma:

\begin{lemma}{\citep{blanchard2007statistical}} \label{lem:blanchard}
Suppose Assumption \ref{assum:kernel1} and \ref{assum:eigendef} are verified.
     For all projector $P$ of rank $l$, with probability at least $1-e^{-\xi}$:
     \begin{equation}
         R_n(P) \leq \frac{3}{2} R(P) + 24 \sqrt{\frac{l}{n}(1-\STL\Sigma_\mu \STR^2)} + \frac{25\xi}{n}
     \end{equation}
     
\end{lemma}
This comes from their equation (30), with $M=1$ in our case, using $K=2$ and bounding their $\rho(M,l,n)$ by $\sqrt{\frac{l}{n} tr C_2'}$ they mentioned as they suggested (where they described $C_2'=\int_\mathcal{X} (\varphi(x) \otimes \varphi(x)^*)^*\otimes(\varphi(x) \otimes \varphi(x)^*) d\mu(x) - \Sigma_\mu \otimes \Sigma_\mu$). For us, the only important point is that it is a bounded quantity.

\newtheorem*{thm:rates}{Theorem~\ref{thm:rates}}
\begin{thm:rates}
Suppose Assumption \ref{assum:kernel1} and \ref{assum:eigendef} are verified.
\begin{itemize}
  
    \item     If the eigenvalues of $\Sigma_\mu$ follow a polynomial decay rate of order $\alpha>1$ (Assumption~\ref{AssumPoly}), then: 
    \begin{equation*}
        \dist(\mu,\mu_n) \lesssim_{\mu^{\otimes n}} 
        n^{-\frac{1}{2}+\frac{1}{2\alpha}}.
    \end{equation*}
    \item   If the eigenvalues of $\Sigma_\mu$ follow an exponential decay rate (Assumption~\ref{AssumExp}), then:
    \begin{equation*}
        \dist(\mu,\mu_n) \lesssim_{\mu^{\otimes n}} 
        \frac{(\log n)^\frac{3}{2}}{\sqrt{n}}.
    \end{equation*}
\end{itemize}
\end{thm:rates}

\begin{proof}
    By the triangular inequality:
\begin{align}
    ||\Sigma-\Sigma_n||_1 &\leq 
    ||\Sigma - P^l(\Sigma)\Sigma ||_1 
    + ||(P^l(\Sigma) -P^l(\Sigma_n))\Sigma ||_1 
    + ||P^l(\Sigma_n)(\Sigma -\Sigma_n) ||_1
    +||P^l(\Sigma_n)\Sigma_n - \Sigma_n ||_1 \notag \\
    &\leq R(P^l(\Sigma)) + \sqrt{2l}  ||(P^l(\Sigma) -P^l(\Sigma_n))\Sigma ||_2 + \sqrt{l} ||(\Sigma -\Sigma_n) ||_2 + R_n(P^l(\Sigma_n)) \notag \\ 
    &\leq R(P^l(\Sigma)) + \sqrt{2l}  (||(P^l(\Sigma)\Sigma -\Sigma ||_2 + ||P^l(\Sigma_n)\Sigma -\Sigma ||_2) + \sqrt{l} ||(\Sigma -\Sigma_n) ||_2 + R_n(P^l(\Sigma_n)) \label{eq:gigaineq}
\end{align}
where we have mainly used  Schatten norm ``Hölder" inequality between 1-norm and 2-norms (eq.~(5)) as well as
$||\Sigma - P^l(\Sigma)\Sigma ||_1 = \sum_{i>l} \lambda_i = R(P^l(\Sigma))$ and similarly for $\Sigma_n$. Now, thanks to Lemma~\ref{lem:decay}, for the polynomial case:
\begin{itemize}

    \item \mbox{}\vspace{-\baselineskip} \begin{equation*}
            R(P^l(\Sigma)) = \Theta \left( n^{-\theta(1-\frac{1}{\alpha})}\right).
    \end{equation*}
    
    \item By Lemma~\ref{lem:blanchard}, 
    \begin{align*}
            R_n(P^l(\Sigma_{n})) &\leq R_n(P^l(\Sigma))  \\
    &\lesssim_{\mu^{\otimes n}} R(P^l(\Sigma)) + \sqrt{\frac{l}{n}} + \frac{1}{n} \\
    &\lesssim_{\mu^{\otimes n}} \max({R(P^l(\Sigma))}, \sqrt{\frac{l}{n}})\\
    &\lesssim_{\mu^{\otimes n}} \max(n^{-\theta(1-\frac{1}{\alpha})},n^{-\frac{1}{2}+\frac{\theta}{2\alpha}}).
    \end{align*}
    
    \item \mbox{}\vspace{-\baselineskip} \begin{align*}
        \sqrt{2l}  ||(P^l(\Sigma)\Sigma -\Sigma ||_2 &\lesssim_{\mu^{\otimes n}} n^{\frac{\theta}{2\alpha}}n^{-\theta+\frac{\theta}{2\alpha}}\\
        &= n^{-\theta+\frac{\theta}{\alpha}}.
    \end{align*}

    \item \mbox{}\vspace{-\baselineskip} \begin{align*}
        \sqrt{2l}  ||(P^l(\Sigma_n)\Sigma -\Sigma ||_2 &\lesssim_{\mu^{\otimes n}} n^{\frac{\theta}{2\alpha}}\max(
        n^{-\frac{1}{2}+\frac{1}{4\alpha}},n^{-\theta+\frac{1}{4\alpha}})\\
        &= \max(n^{-\frac{1}{2}+\frac{1}{4\alpha}+\frac{\theta}{2\alpha}}, n^{-\theta+\frac{\theta}{2\alpha}+\frac{1}{4\alpha}}).
    \end{align*}

    \item From MMD convergence rates, we know:
    \begin{align*}
        \sqrt{l} ||(\Sigma -\Sigma_n) ||_2 \lesssim_{\mu^{\otimes n}} \sqrt{\frac{l}{n}}.\\
        = n^{-\frac{1}{2}+\frac{\theta}{2\alpha}}
    \end{align*}
\end{itemize}
Now combining all those terms, from eq.~(\ref{eq:gigaineq}) we get, for $0<\theta\leq\alpha$:
\begin{align*}
     ||\Sigma-\Sigma_n||_1  
    &\lesssim_{\mu^{\otimes n}} n^{-\max\left(
     -\theta+\frac{\theta}{\alpha},
     -\frac{1}{2}+\frac{\theta}{2\alpha},
     -\frac{1}{2}+\frac{1}{4\alpha}+\frac{\theta}{2\alpha}, 
     -\theta+\frac{\theta}{2\alpha}+\frac{1}{4\alpha} \right)}\\
        &\lesssim_{\mu^{\otimes n}} n^{-\max\left(
     -\theta+\frac{\theta}{\alpha},
     -\frac{1}{2}+\frac{1}{4\alpha}+\frac{\theta}{2\alpha}, 
     -\theta+\frac{\theta}{2\alpha}+\frac{1}{4\alpha} \right)}.
\end{align*}
The decreasing lines $-\theta+\frac{\theta}{\alpha}$  and $ -\theta+\frac{\theta}{2\alpha}+\frac{1}{4\alpha}$ and the increasing line $-\frac{1}{2}+\frac{1}{4\alpha}+\frac{\theta}{2\alpha}$ all cross at $\theta=\frac{1}{2}$, so we find it is the optimal $\theta$ to minimise, which gives a rate of $n^{-\frac{1}{2}+\frac{1}{2\alpha}}$.

Similarly, we do the same for the exponential decay rate case, in particular if we take $\theta=1$ we get:
\begin{equation*}
    ||\Sigma-\Sigma_n||_1 \lesssim_{\mu^{\otimes n}} \frac{(\log n)^\frac{3}{2}}{\sqrt{n}}
\end{equation*}
where the dominating term in eq.~(\ref{eq:gigaineq}) is $\sqrt{2l}  ||(P^l(\Sigma_n)\Sigma -\Sigma ||_2$.
\qed
\end{proof}

\subsubsection{Robustness properties (proof of subsection~\ref{subsec:robust})}
\label{appendsssec:robust}

\newtheorem*{prop:robust}{Proposition \ref{prop:robust}}
\begin{prop:robust}
        Denote $P_\varepsilon = (1-\varepsilon) P + \varepsilon C$ where $C$ is some contamination distribution.
        We have when Assumption~\ref{assum:kernel1} is verified: $|\dist(P_\varepsilon,Q)-\dist(P,Q)| \leq 2 \varepsilon$.
\end{prop:robust}

\begin{proof}
    Since $\dist$ is a metric based on a norm: 
    \begin{align*}
        |\dist(P_{\varepsilon},Q)-\dist(P,Q)|&\leq \dist(P,P_{\varepsilon}) \\
        &= ||\varepsilon(\Sigma_P-\Sigma_{C})||_1\\
        &= \varepsilon ||(\Sigma_P-\Sigma_{C})||_1 \\
        &\leq 2\varepsilon
    \end{align*}
    \qed
\end{proof}

The proof works for the Schatten 2-norm ($MMD_{k^2}$) as well.

On the contrary in $(\mathbb{R}^d,||\cdot||)$ for the classical Wasserstein 1-distance, if for distribution $Q$, we have $1-\frac{\varepsilon}{2}$ mass contained in a ball of some center $c$ and some radius $r$, then taking as contamination $C=\delta_x$ a Dirac in some point $x$, we must have $W_1(P_{\varepsilon},Q)\geq \frac{\varepsilon}{2}||x-c||-r$ and we can take $||x||\to\infty$ to make the distance diverge.

\subsection{Proof of \Cref{prop:computation}}\label{sec:proof_computation}

\begin{proof}
Notice that, as $ \Tilde{\varphi}(z_k) \Tilde{\varphi}(z_k)^* = (\mu_n-\nu_m)(z_k)\varphi(z_k) \varphi(z_k)^*$, we have $\Sigma_{\mu_n-\nu_m}= \sum_{k=1}^r\Tilde{\varphi}(z_k) \Tilde{\varphi}(z_k)^* = ZZ^*$.
Since $\Sigma_{\mu_n-\nu_m}$ is real symmetric, it is diagonalisable, and as a sum of $r$ projectors, it is of rank $r$ (by linear independence of the $(\varphi(z_k))_{i=1...r}$). Let us denote $(\lambda_k,v_k)_{k=1,..,r}$ the couples of eigenvalues and eigenvectors of the restriction of $\Sigma_{\mu_n-\nu_m}$ on the subspace of dimension $r$ spanned by those projectors. Those eigenvectors are orthogonal and those eigenvalues are non-zero. 
Note that for such an eigen-(value,vector) $(\lambda,v)$ of $ZZ^*$, we have: $Z^* ZZ^* v = Z^*(\lambda v)$, therefore $Z^* v$ is an eigenvector of $Z^* Z$ with associated eigenvalue $\lambda$ as well.
Then note that $(Z^*v_k)_{k=1,..,r}$ are also orthogonal and their norm is not zero, since $\langle Z^*v_i, Z^* v_j \rangle =v_i^*ZZ^*v_j=v_i^*\lambda_j v_j = \lambda_j \langle v_i, v_j \rangle$. Therefore they are distinct, and we can fully diagonalise $K$ (which is of size $r\times r$) with such vectors.
\qed
\end{proof}

\section{EXPERIMENTS}\label{appendsec:experiments}
Computation were carried on a Macbook Pro 2020 with processor 2,3 GHz Intel Core i7 (4 cores) and memory 16 Go 3733 MHz LPDDR4X (graphic card is Intel Iris Plus Graphics 1536 Mo).

\subsection{Normalised energy (Additional figures for subsection~\ref{subsec:mmdpitfalls})}
\label{appendssec:XPmmdpitfalls}

Here are some other simulations to highlight the different geometrical behaviours of MMD and $d_{KT}$: in Figure~\ref{fig:var_mean},
we compute the two distances (both with Gaussian kernel with $\sigma=1$) for two sets of $n=1000$ samples, one following a standard normal law $\mathcal{N}(0,1)$ and the other following $\mathcal{N}(\theta,1)$ and we try for different values of $\theta$ from $0$ to $10$ with steps of $0.5$. As can be seen, even when the two locations are far apart, since the two distributions are not Dirac, their variances prevent MMD to reach the maximum 2. This means that when the distance is getting close to zero in general the slope will be flatter for MMD than for $d_{KT}$: one should be aware of that when doing for instance a gradient descent. We also added the Kernel Bures-Wasserstein (computed with a simplified kernel trick as in~\citet{KWD}) to illustrate the Fuchs-van de Graaf inequalities. For $\theta$ high enough, the lower bound joins $\dist$ (they naturally cannot go higher than 2 because of Assumption~\ref{assum:kernel1}).

In Figure~\ref{fig:var_std} 
this time, the same experiment is shown but this time between distributions $\mathcal{N}(0,s)$ and $\mathcal{N}(100,s)$, where we make the variance $s$ varies in $[0.1, 0.3, 1, 3, 10, 30, 100]$. As the variance grows, the Hilbert norm of the distributions decreases and so the MMD decreases very fast. On the contrary, $d_{KT}$ takes its maximum value 2 distinguishing the far-away distributions and starts only to decreases for high values of $s$ down to close to 1 for $s=100$ (which the distance between the two locations), which seems reasonable as, for low variance values, the support of the majority of the central masses of each distributions do not even overlap!







\begin{figure}[!htb]
\centering
\begin{minipage}{0.7\textwidth}
     \centering
     
     \includegraphics[width=\linewidth]{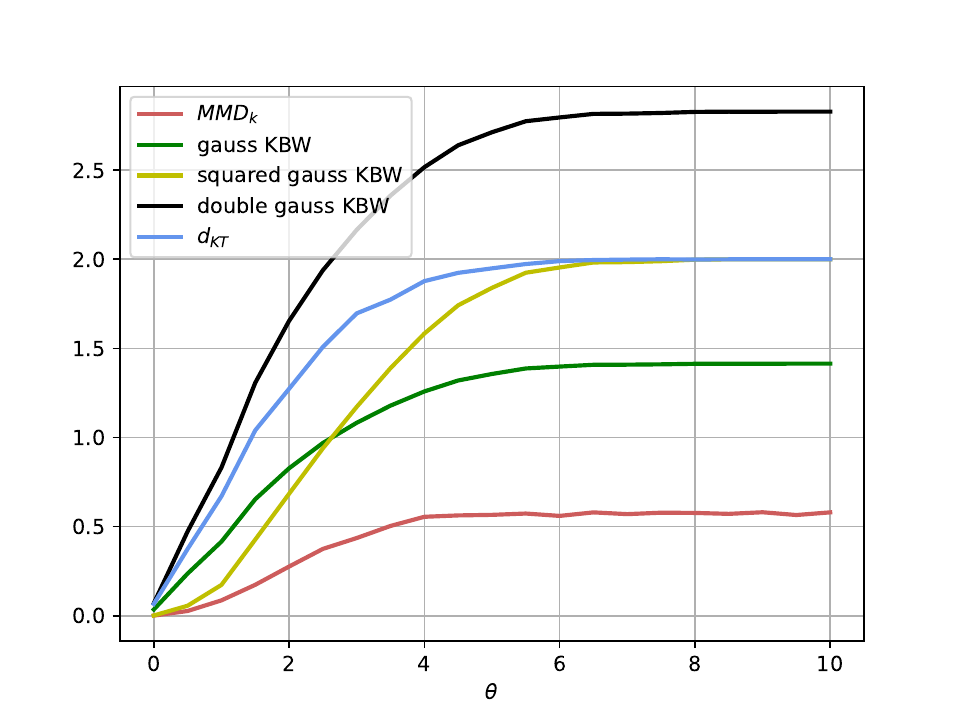}
         \caption{Variations on the mean $\theta$}
         \label{fig:var_mean}

\end{minipage}

\begin{minipage}{0.7\textwidth}

     \centering
     \includegraphics[width=\linewidth]{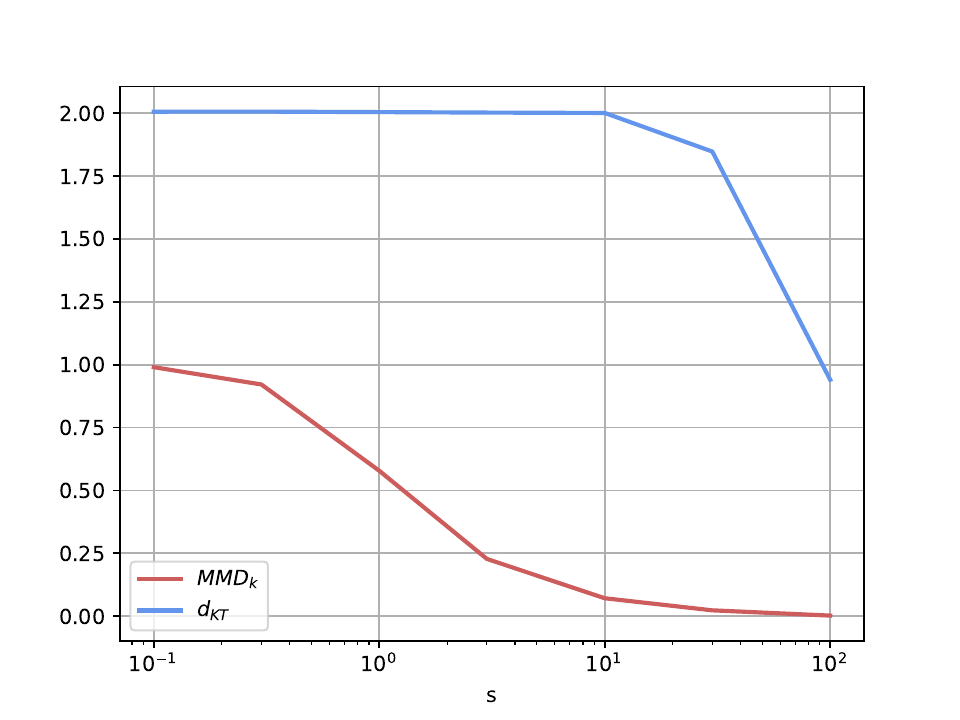}
         \caption{Variations on the standard deviation $s$}
         \label{fig:var_std}

\end{minipage}

\end{figure}

\subsection{ABC (Additional table and figures)}
\label{appendssec:XPABC}

In Table~\ref{tab:abcsupp}, we display more complete results of the experiments adding other methods as well as the standard deviations over the different runs. 
Referring to the distances of Corollary~\ref{cor:wasserineq}, we added the Optimal Transport (OT) 1-Wasserstein distance with Euclidean cost (as we mentioned it rejects all samples because of the contamination), as well as the one with a distance cost based on the Gaussian kernel $c_k$ (noted as $\mathrm{OT}_{gauss}$ in the table) and the Kernel Bures-Wasserstein distance ($d_{KBW}$). 
We added some ``normalised" MMD: 
\begin{equation*}
    MMD_N(\mu,\nu) =\frac{\SL{2} \Sigma_\mu - \Sigma_\nu \SR{2} }{ \sqrt{\SL{2} \Sigma_\mu\SR{2}^2 + \SL{2} \Sigma_\nu\SR{2}^2}}
\end{equation*} 
from the inequality of Section~\ref{subsec:mmdpitfalls} in order to attenuate its effect. It is not guaranteed to be a distance, and eventually it does not perform better than $\dist$.
For the sake of fairness, we compare ourselves to competitors that also require cubic complexity of computation such as the Kernel Fischer Discriminant Analysis (KFDA~\citep{KFDA}) with parameter $\gamma_n = n^{-1/2} = 0.1$ but it also rejects all the time, so we added its normalised version as well.

\begin{table}[!htb]
\caption{Average MSE of ABC Results. \\ 
The Optimal Transport Wasserstein distance (OT) has been added, which rejects all samples.
The KFDA is used with $\gamma_n = n^{-1/2} = 0.1$.
\label{tab:abcsupp} }
\centering
\begin{tabular}{llrr}
\toprule
 $\varepsilon$ & distance & \#accept. (\textsl{std}) & MSE (\textsl{std})  \\
\toprule
\multirow{5}{*}{0.05} & OT & 0  &  N/A   \\
& KFDA & 0  &  N/A   \\
& $\mathrm{KFDA}_{norm.}$ & 1457 (\textsl{123})  &  0.41 (\textsl{0.05})  \\
& $\mathrm{MMD}$ & 1092 (\textsl{45}) &  0.19 (\textsl{0.02})   \\
& $\mathrm{MMD_N}$  & 0 &  N/A   \\
& $\mathrm{MMD_E }$ & 0 &  N/A   \\
& $d_{KBW}$  & 0 & N/A   \\
& $\mathrm{OT}_{gauss}$  & 0 & N/A   \\
                    & $\dist$  & 0 & N/A   \\
 \midrule
\multirow{5}{*}{0.25}  & OT & 0  &  N/A   \\
& KFDA & 0  &  N/A   \\
& $\mathrm{KFDA}_{norm.}$ & 1557 (\textsl{122})  &  0.45 (\textsl{0.05})  \\
& $\mathrm{MMD}$ & 2964 (\textsl{92}) & 1.29 (\textsl{0.06})    \\
& $\mathrm{MMD_N}$  & 840 (\textsl{30}) &  0.12 (\textsl{0.01})   \\
& $\mathrm{MMD_E }$ & 0 &  N/A   \\
& $d_{KBW}$  & 0 & N/A   \\
& $\mathrm{OT}_{gauss}$  & 343 (\textsl{48})  & 0.04 (\textsl{0.01})    \\
                    & $\dist$  & 58 (\textsl{25}) & \textbf{0.03} (\textsl{0.01})   \\
\midrule
\multirow{5}{*}{0.5}  & OT & 0  &  N/A   \\
& KFDA & 0  &  N/A   \\
& $\mathrm{KFDA}_{norm.}$ & 1673 (\textsl{118})  &  0.49 (\textsl{0.05})  \\
& $\mathrm{MMD}$ & 6168 (\textsl{406}) & 7.47 (\textsl{1.83})   \\
& $\mathrm{MMD_N}$  & 1964 (\textsl{69})  &  0.57 (\textsl{0.02})   \\
& $\mathrm{MMD_E }$ & 846 (\textsl{35})  &  0.17 (\textsl{0.05})    \\
& $d_{KBW}$  & 1312 (\textsl{49}) & 0.26 (\textsl{0.02})  \\
& $\mathrm{OT}_{gauss}$  & 1376 (\textsl{53})  & 0.29 (\textsl{0.02})   \\     
                    & $\dist$  & 828 (\textsl{34}) & 0.12 (\textsl{0.01})   \\
\midrule
\multirow{5}{*}{1}   & OT & 0  &  N/A   \\
& KFDA & 0  &  N/A \\
& $\mathrm{KFDA}_{norm.}$ & 1847 (\textsl{121})  &  0.57 (\textsl{0.05})  \\
& $\mathrm{MMD}$ & 10000 (\textsl{0}) &  26.0 (\textsl{0.18})   \\
& $\mathrm{MMD_N}$  & 9488 (\textsl{57})  &  20.4 (\textsl{0.31})    \\
& $\mathrm{MMD_E }$ & 2926 (\textsl{52})  &  1.33 (\textsl{0.6})    \\
& $d_{KBW}$  & 3709 (\textsl{54}) & 2.02 (\textsl{0.05})  \\
& $\mathrm{OT}_{gauss}$  & 3484 (\textsl{84})  & 1.78 (\textsl{0.06})   \\
& $\dist$  & 2067 (\textsl{93}) & 0.63 (\textsl{0.04})   \\
\bottomrule
\end{tabular}
\end{table}

In Figure~\ref{fig:posteriors} are displayed the simulated posteriors $\frac{1}{|L_\theta|} \sum_{\theta_i \in L_\theta} p(\cdot|\theta_i)$ obtained as a result of Rejection ABC Algorithm~\ref{alg:reject_abc} using the Gaussian kernel for both MMD and $\dist$. For sensible parameter $\varepsilon$, the posteriors obtained via $\dist$ are quite close to the target, while for for MMD (gaussian) the posterior stays very flat like the prior unless $\varepsilon$ goes to a very low value (less than the contamination threshold).
For sake of visibility, the best posteriors of normalised MMD and MMD with the energy kernel are displayed in another Figure~\ref{fig:otherposteriors}. For the energy kernel, we can see that the peak of the density is not aligned with the one of the target.

\begin{algorithm}[!htb]                      
                \caption{Rejection ABC Algorithm}
                \label{alg:reject_abc}
                \footnotesize
                \begin{algorithmic}[1]
                        \REQUIRE {Observed data $\{X_{i}\}_{i=1}^{n}$, prior $\pi(\theta)$ on the parameter space $\Theta$, tolerance threshold $\epsilon$, statistical distance $d$, empty list $L_\theta$}
                        \FOR {$i=1$ to $T$}
                        \STATE draw $\theta_i \sim \pi(\theta)$ 
                        \STATE draw $Y_{1} \ldots,Y_{m} \overset{\textrm{i.i.d.}}{\sim} p_{\theta_i}$
                        \IF {$d(X^{n}, Y^{m}) < \epsilon$}
                            \STATE  Add $\theta_i$ to $L_\theta$
                        \ENDIF
                        \ENDFOR
                        \RETURN $L_\theta$
                \end{algorithmic}
\end{algorithm}

\begin{figure}[!htb]
     \centering
     \begin{minipage}{.7\textwidth}
     \centering
     \includegraphics[width=\linewidth]{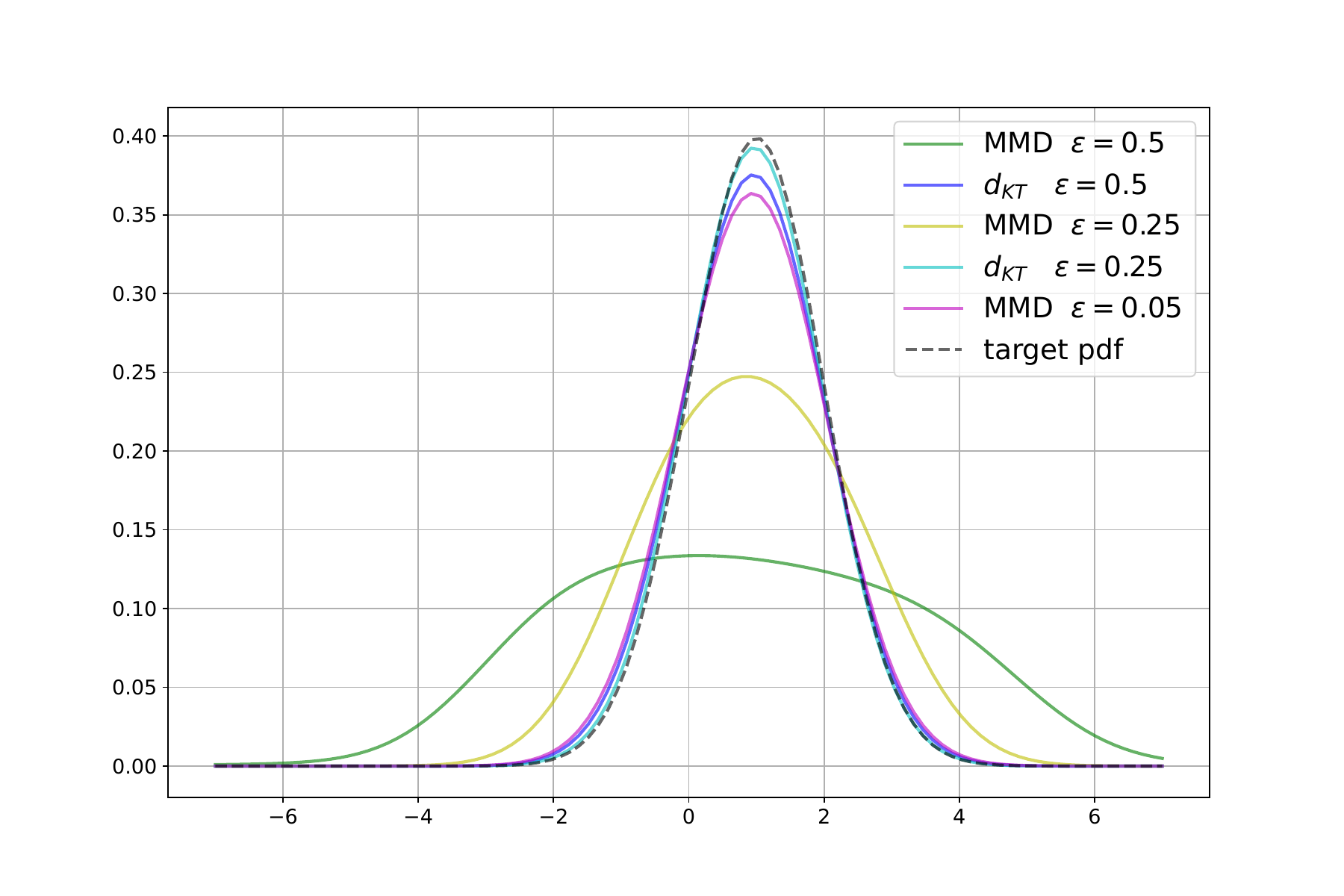}
         \caption{Posterior probability density functions using Gaussian kernel}
         \label{fig:posteriors}
    \end{minipage}
    \begin{minipage}{.7\textwidth}
     \centering
     \includegraphics[width=\linewidth]{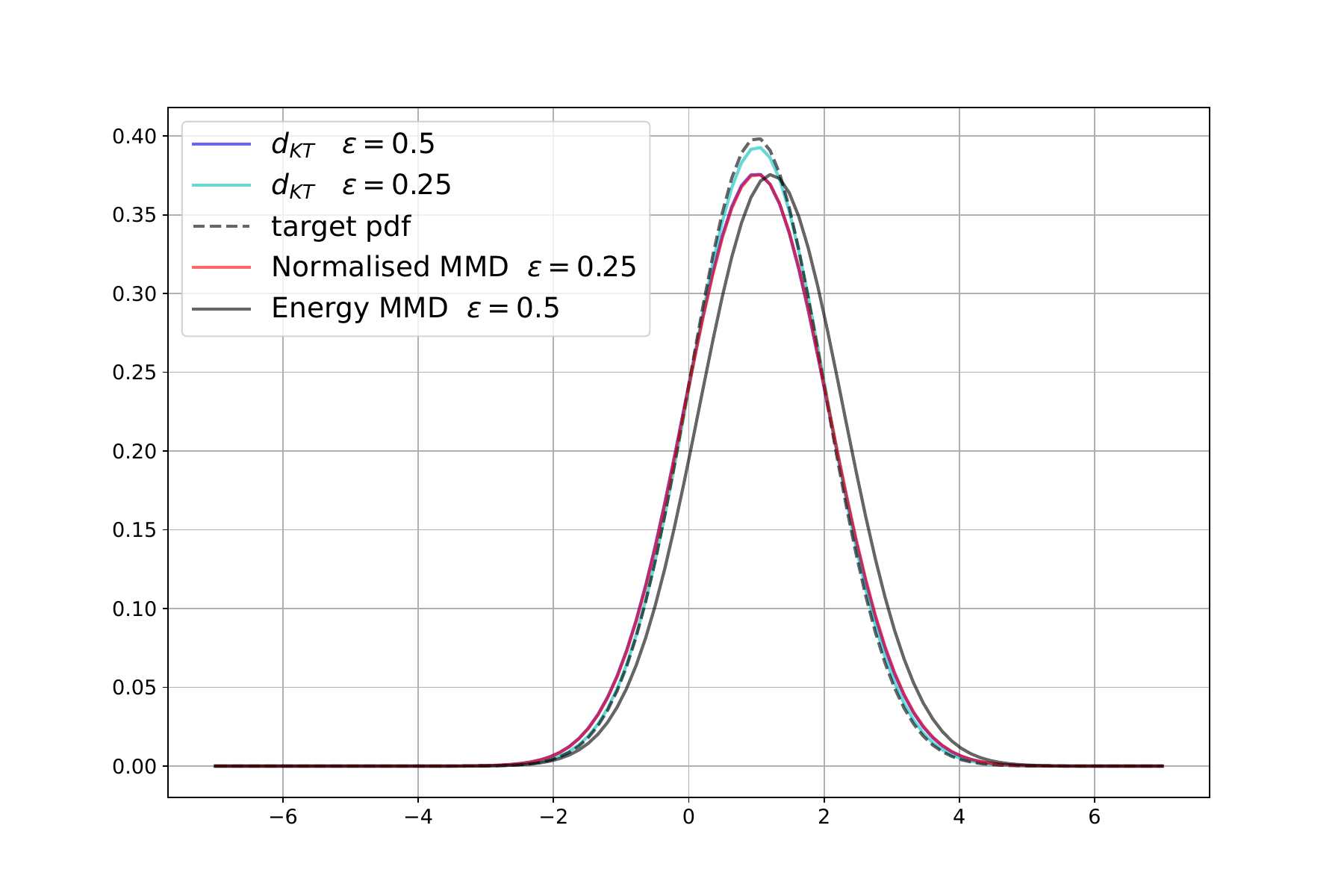}
         \caption{Posterior probability density functions of $\dist$ and other competitors than classic MMD Gaussian kernel}
         \label{fig:otherposteriors}
    \end{minipage}
\end{figure}




\subsection{Particle gradient flows}\label{appendssec:XPGF}
Here are displayed at different iterations the particle flows of $\dist$ (Fig.~\ref{fig:flowd1}) and MMD (Fig~\ref{fig:flowmmd}).
We choose the Laplacian kernel over the Gaussian kernel, as it gave a better convergence results for both $\dist$ and MMD. Even though it is not differentiable at the coordinates of the target particles, in practice computationally we observed that no problem occured as this typically would happen when the cloud of points are reaching destination. 




\begin{figure}[!htb]
     \centering
     \begin{minipage}{.7\textwidth}
     \centering
     \includegraphics[width=\linewidth, trim={0 0.3cm 0 0.3cm},clip]{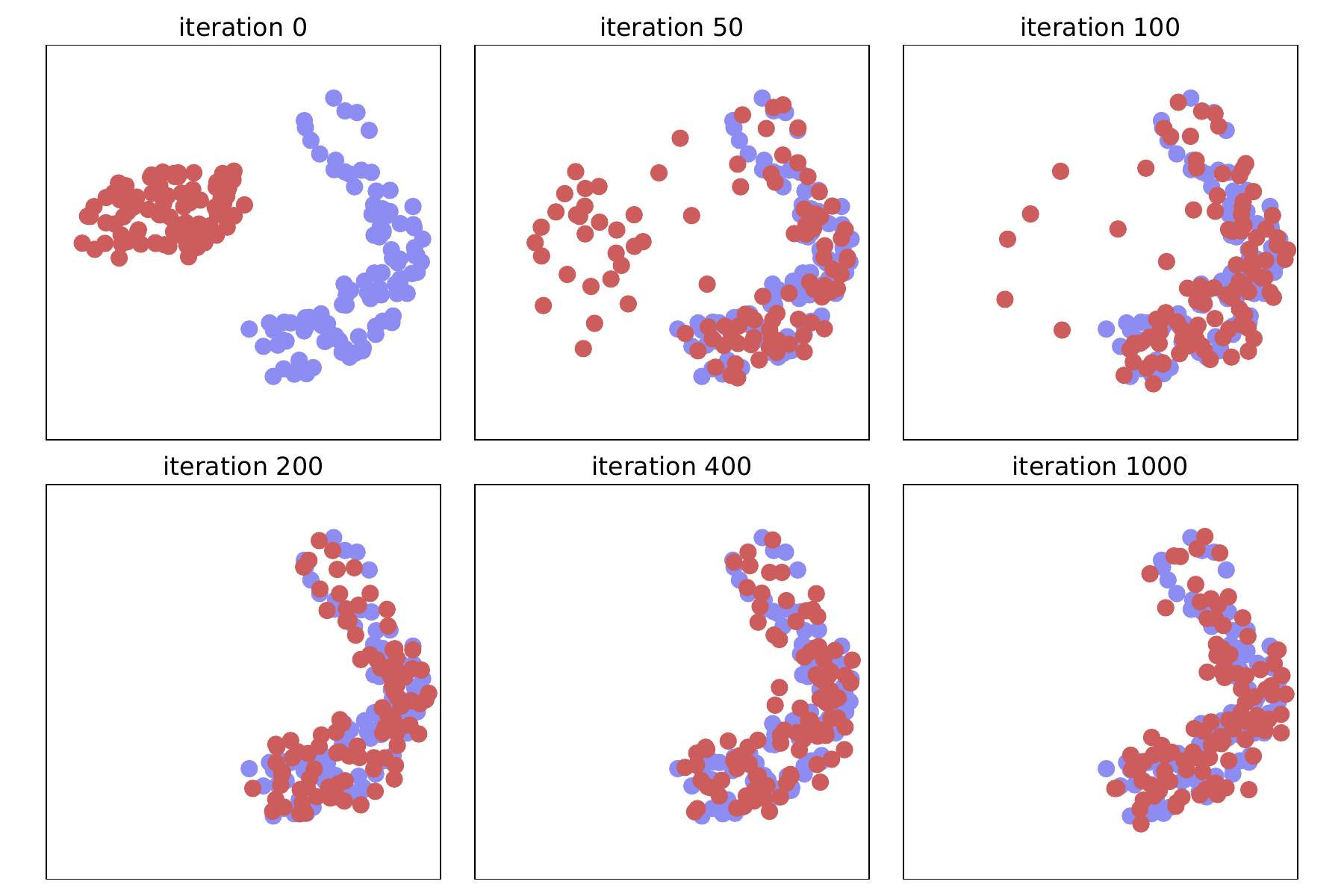}
         \caption{Particle flow with $\dist$ leads to a good match between the distributions}
         \label{fig:flowd1}
    \end{minipage}    
    \par\vskip 1cm
    \begin{minipage}{.7\textwidth}    
     \centering
     \includegraphics[width=\linewidth]{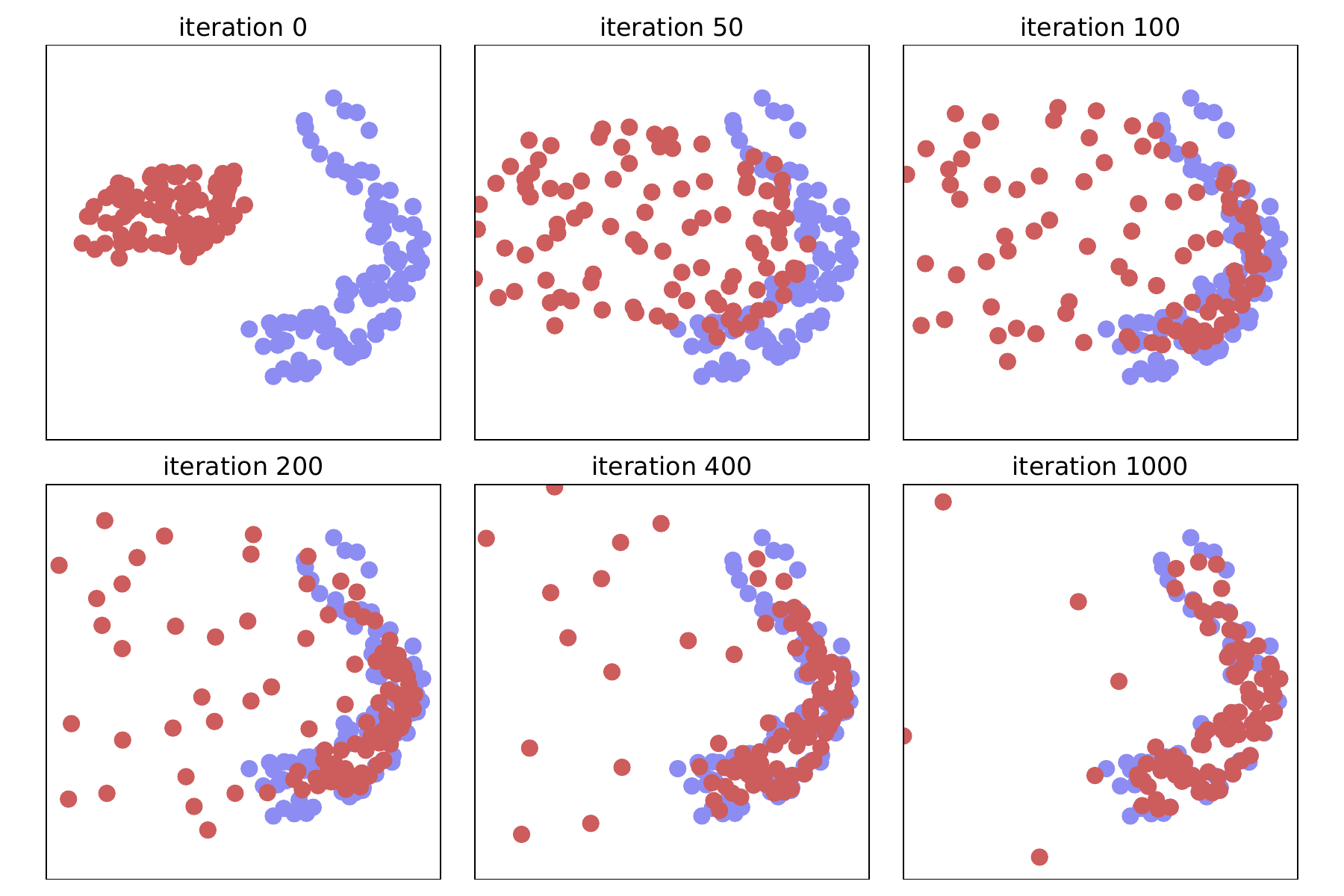}
         \caption{Particle flow with MMD leads to several samples being ``repulsed" due to internal energy.
         \label{fig:flowmmd}}
    \end{minipage}
\end{figure}

We also add a shape transfer task as in~\citet{chazal2024statistical} displayed in Figure~\ref{fig:shapetransfer}, where we added to their results our $\dist$ using the Laplacian kernel with bandwidth $\sigma=1$ again and with a learning rate of 5. The other methods comprises the $KKL_\alpha$ with $\alpha=0.01$ which is a regularised version of the KKL which uses the same RKHS density operators as us but using the Von Neumann quantum relative entropy and is also in cubic time complexity, and KALE with parameter $\lambda = 0.001$ which approaches the classic KL divergence, and $\lambda = 10000$ which approaches MMD (and thus we call it so in the figure). 

\begin{figure}[!htb]
     \centering
     \includegraphics[width=\linewidth]{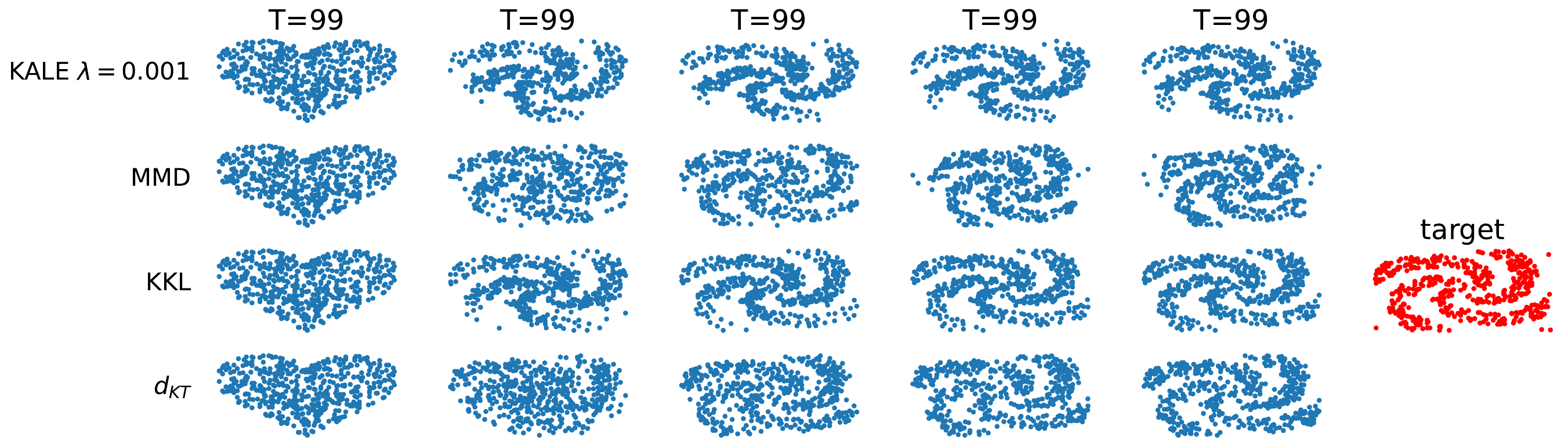}
         \caption{Shape transfer}
         \label{fig:shapetransfer}
\end{figure}


\end{document}